\documentclass{article}
\usepackage{iclr2025_conference,times}
\iclrfinalcopy

\usepackage{times}
\usepackage{hyperref}
\usepackage{url}
\usepackage{comment}   
\usepackage{float}
\usepackage{longtable}
\usepackage{colortbl}
\usepackage{UserDefinedCommands}
\usepackage{footnote}
\makesavenoteenv{figure*}

\definecolor{darkblue}{RGB}{0, 0, 210}
\newcommand{\rp}[1]{{\color{darkblue}#1}}

\title{A Shared Low-Rank Adaptation Approach to Personalized RLHF}

\author{
Renpu Liu\,,~~Peng Wang\,,~~Donghao Li\,,~~Cong Shen\,,~~Jing Yang\thanks{Corresponding author.}\\ 
University of Virginia\\ 
\texttt{\{pzw7bx, pw7nc, maj3qx, cong, yangjing\}@virginia.edu}}

\begin{document}
\maketitle

\begin{abstract}

Reinforcement Learning from Human Feedback (RLHF) has emerged as a pivotal technique for aligning artificial intelligence systems with human values, achieving remarkable success in fine-tuning large language models. However, existing RLHF frameworks often assume that human preferences are relatively homogeneous and can be captured by a single, unified reward model. This assumption overlooks the inherent diversity and heterogeneity across individuals, limiting the adaptability of RLHF to personalized scenarios and risking misalignments that can diminish user satisfaction and trust in AI systems. In this paper, we address these challenges by introducing Low-Rank Adaptation (LoRA) into the personalized RLHF framework. We apply LoRA in the the aggregated parameter space of all personalized reward functions, thereby enabling efficient learning of personalized reward models from potentially limited local datasets. Our approach exploits potential shared structures among the local ground-truth reward models while allowing for individual adaptation, without relying on restrictive assumptions about shared representations as in prior works. We further establish sample complexity guarantees for our method. Theoretical analysis demonstrates the effectiveness of the proposed approach in capturing both shared and individual-specific structures within heterogeneous human preferences, addressing the dual challenge of personalization requirements and practical data constraints. Experimental results on real-world datasets corroborate the efficiency of our algorithm in the personalized RLHF setting.

\end{abstract}

\section{Introduction}

The rapid development and widespread use of Large Language Models (LLMs) have transformed fields like natural language processing, content generation, and human-computer interaction. Models such as GPT-4~\citep{achiam2023gpt}, BERT~\citep{devlin2019bert}, and their successors have exhibited remarkable capabilities in understanding and generating text, enabling applications ranging from automated customer service to advanced creative tools. This ``boom'' of LLMs has not only broadened AI’s potential but also underscored the critical need to ensure that these models align with human values and preferences.

To help this alignment, Reinforcement Learning from Human Feedback (RLHF)~\citep{ouyang2022training,christiano2023deepreinforcementlearninghuman} plays a key role as a fine-tuning method of LLMs. This method ensures that the generated responses are contextually appropriate and aligned with ethical and social norms~\citep{ouyang2022training}. By incorporating human feedback into the fine-tuning process, RLHF bridges the gap between the raw generative power of LLMs and the requirements of real-world applications, improving the quality and safety of AI-generated content.

Current RLHF frameworks, such as \cite{bai2022training,wang2024rlhf}, essentially assume that human preferences are relatively homogeneous and can be effectively captured by a single, unified reward model. This simplification overlooks the inherent diversity and heterogeneity in human preferences, which can vary significantly across individuals. Such an oversimplification limits the adaptability of RLHF to personalized scenarios and risks, introducing misalignments that could diminish user satisfaction and trust in AI systems. A straightforward approach to handling heterogeneous human preferences is learning personalized reward functions for each labeler using traditional RLHF methods, such as \cite{ouyang2022training}. However, this method faces a significant challenge: preference data from individual users may be insufficient to construct accurate reward models for each human labeler. 
Recently, several studies have proposed empirical methods to address this challenge. For example, \citet{li2024personalized} introduced a personalized direct preference optimization method within the personalized RLHF framework. Similarly, \citet{poddar2024personalizing} presented a class of multi-modal RLHF methods that infer user-specific latent variables and then learn personalized reward models conditioned on them. In addition to empirical approaches, some works have provided methods with theoretical guarantees. Specifically, \citet{zhong2024provable} conducted a theoretical analysis assuming that human reward functions are linear with shared representations. Extending this line of work, \citet{park2024rlhf} considered a more general setting where the representation function is a general (nonlinear) function of the feature mapping.

On the other hand, since first introduced by \citet{hu2021lora}, Low-Rank Adaptation (LoRA) has quickly become a prominent method for fine-tuning LLMs to reduce the number of trainable parameters and prevent overfitting~\citep{houlsby2019parameter, huang2023lorahub}. Some recent works have proposed to combine RLHF with LoRA to enhance the fine-tuning of LLMs using human feedback. For instance, researchers have explored integrating LoRA into the RLHF framework to efficiently incorporate human preferences while maintaining model performance \citep{santacroce2023efficient, sun2023exploring, sidahmed2024perl}. However, these approaches primarily focus on general adaptation and do not address the challenges of heterogeneous feedback from diverse users.

In this paper, we address the challenges of personalized RLHF by introducing personalized LoRA with a shared component into the personalized RLHF framework. By leveraging LoRA, we effectively learn individual reward models that capture human users' heterogeneous preferences with limited data. To the best of our knowledge, LoRA has not been previously explored in the context of personalized RLHF, making our approach a novel contribution to the field. Our major contributions are summarized as follows: 

\begin{itemize}[leftmargin=*]\itemsep=0pt 
    \item We propose an algorithm named {Personalized LoRA with Shared Component} (\texttt{P-ShareLoRA}) for RLHF, which leverages the shared components of LoRA modules to learn the personalized reward functions efficiently. Rigorous theoretical analysis demonstrates that \texttt{P-ShareLoRA} can effectively reduce sample complexity, compared with both the full-parameter fine-tuning method and the standard LoRA method without parameter sharing. To the best of our knowledge, this is the first work that 
    theoretically demonstrates the benefits of LoRA with shared components in RLHF.

 \item 
 Unlike existing analytical frameworks for personalized RLHF which typically enforce strict constraints on the reward model structures, such as linear representations~\citep{zhong2024provable} or shared representations with linear heads~\citep{park2024rlhf}, we develop novel technical approaches to address the challenges from the unstructured reward functions. Specifically, we propose a new Lagrange remainder-based method that allows us to prove that LoRA modules with shared components can approximate the optimal low-rank structure of the ground truth parameter matrix. Building on this, we further prove an upper bound on the distance between the optimal reward function and the learned reward function with shared parameters. The theoretical results demonstrate that the expected return under the policies derived with the learned reward functions are near-optimal (up to a bias term related to the preference diversity among users).

    \item Experiments on the Reddit TL;DR dataset~\citep{stiennon2020learning} validate the effectiveness of the proposed approach. 
    Specifically, our approach achieves a prediction accuracy of 74.65\% on Llama-3 8B and 66.93\% on GPT-J 6B, which outperforms the SOTA algorithms that achieve 73.25\% on Llama-3 8B and 66.13\% on GPT-J 6B, respectively. 
    Those empirical results corroborate our theory, demonstrating the advantage of LoRA with shared components for personalized RLHF. 
\end{itemize}

\section{Related Works}

\paragraph{Reinforcement Learning from Human Feedback.}

Reinforcement Learning from Human Feedback (RLHF) has demonstrated considerable success across various practical applications, especially in aligning AI models with human values and preferences. One of the most prominent applications of RLHF is in fine-tuning large language models, as exemplified by OpenAI’s ChatGPT~\citep{ouyang2022training} and GPT-4~\citep{achiam2023gpt}. Additionally, RLHF has been explored in computer vision tasks~\citep{lee2023aligning,xu2024imagereward}. Furthermore, RLHF has been widely adopted in domains that involve high-risk decision-making, such as healthcare~\citep{yu2021reinforcement}, robotics~\citep{abramson2022improving, hwang2023promptable,thumm2024text2interaction}, and autonomous driving~\citep{wu2023toward,chen2023feedback}, where alignment with human preferences is critical for ensuring safety and addressing ethical considerations.

From a theoretical standpoint, studies of RLHF have garnered increasing research interest. \citet{zhu2023principled} examine the Bradley-Terry-Luce model~\citep{bradley1952rank} within the context of a linear reward framework, while \citet{zhan2023provable} extend these results to more general classes of reward functions. Similarly, \citet{li2023reinforcement} introduce a pessimistic algorithm that is provably efficient for dynamic discrete choice models. All these works focus on settings with offline preference data. In the online setting, \citet{xu2020preference} and \citet{pacchiano2021dueling} study tabular online RLHF. \citet{wang2024rlhf} theoretically demonstrate that preference-based RL can be directly addressed using existing reward-based RL algorithms by utilizing a preference-to-reward model. \citet{xiong2024iterative} present a provable iterative Direct Preference Optimization (DPO) algorithm for online settings. \citet{ye2024online} provide a theoretical analysis of RLHF under a general preference oracle, proposing sample-efficient algorithms for both offline and online settings.

Some recent studies have extended RLHF to personalized alignment for diverse user groups and individuals. \citet{zhao2023group} introduce Group Preference Optimization (GPO), which addresses group-level heterogeneity through a mixture of shared and personalized architectures. Additionally, \citet{ramesh2024group} propose Group Robust Preference Optimization (GRPO), a reward-free RLHF framework that handles heterogeneous preferences by optimizing for worst-case group outcomes. Beyond group-level alignment, other works focus on individual personalization. For instance, \citet{li2024personalized} develop a Personalized RLHF method that jointly learns a lightweight user model alongside the policy model to capture each user’s unique preferences, leading to responses more closely aligned with individual tastes than non-personalized RLHF. Besides, \citet{poddar2024personalizing} introduce a variational latent preference framework that infers a user-specific latent variable on which both the learned reward model and the policy rely.

In addition to the empirical studies, recent works have also established theoretical guarantees for personalized RLHF. \citet{siththaranjan2023distributional} show that traditional RLHF models that implicitly aggregate preferences can lead to undesirable outcomes. They introduce Distributional Preference Learning (DPL) to mitigate this issue.
\citet{chakraborty2024maxmin} group individual reward models into distinct subsets and propose a MaxMin alignment objective inspired by Egalitarian principles. 
\citet{zhong2024provable} investigate a setting where local optimal reward functions share a linear representation combined with personalized linear heads, theoretically demonstrating that aggregating multiple preferences across different parties can overcome the shortcomings of traditional RLHF that only learn a single reward function.
Building on this, \citet{park2024rlhf} generalize the reward function model of \citet{zhong2024provable} by introducing a general representation function combined with personalized linear heads. 

\paragraph{Low-Rank Adaptation (LoRA).} The rapid scaling of pre-trained language models has led to significant challenges in fine-tuning these models for downstream tasks due to the substantial computational and storage requirements. To address this, Low-Rank Adaptation (LoRA) has been proposed as an efficient fine-tuning approach~\citep{hu2021lora}. The vanilla LoRA keeps the original model weights frozen and injects trainable low-rank matrices into each layer of the Transformer architecture. This strategy dramatically reduces the number of trainable parameters and computational overhead, making it feasible to adapt large models on limited hardware resources~\citep{valipour2022dylora,zhang2023adalora,kopiczko2023vera,dettmers2024qlora,hayou2024lora+,yang2024dora}.

Recently, several studies have focused on implementing LoRA in multi-task settings. \citet{huang2023lorahub} introduce LoraHub, which enables the composition and sharing of LoRA modules trained on diverse tasks. \citet{luo2024moelora} consider LoRA as a Mixture of Experts (MoE), treating these small adaptation modules as experts focusing on unique aspects. 
\citet{shen2024multimodal} introduce MixLoRA, treats LoRA modules as experts and uses a dynamic factor selection method to select modules for combination.
\citet{tang2023parameter} propose partial linearization, where they linearize only the adapter modules—the parts adjusted during fine-tuning—and apply ``task arithmetic'' to combine these linearized adapters from different tasks. In the federated learning setting, \citet{wang2024flora} introduces a stacking-based aggregation technique for LoRA adapters, enabling efficient fine-tuning across clients. 

To effectively learn LoRA modules in a multi-task setting, some recent studies consider sharing partial parameters among different tasks or clients. \citet{sun2024improving} introduce FFA-LoRA, which keeps one of the LoRA modules fixed while updating only the other during local training. Similarly, \citet{kuo2024federated} propose a method in which certain parameters within the locally downloaded LoRA modules remain unchanged, while the rest are updated. HydraLoRA~\citep{tian2024hydralora} extends this idea by incorporating LoRA modules with a shared low-rank matrix in a Mixture-of-Experts (MoE) framework. Additionally, FedSA-LoRA~\citep{guo2024selective} observes that in a federated learning setup, one transformation matrix primarily captures generalizable knowledge, while the other learns client-specific adaptations. Building on this insight, they employ a hybrid approach that combines shared global components with personalized local updates. {\it To the best of our knowledge, the theoretical implications of using shared LoRA parameters in RLHF remain unexplored.} 
\section{Problem Formulation}

\paragraph{Notation.}

Bold uppercase letters (e.g., \(\mathbf{X}\)) denote matrices. The function \(\diag(x_1, \dots, x_d)\) represents a \(d \times d\) diagonal matrix with diagonal entries \(x_1, \dots, x_d\). The inner product of vectors \(x\) and \(y\) is denoted by \(\langle x, y \rangle\), and the Euclidean norm of a vector \(x\) is represented by \(\|x\|_2\). For a matrix \(\mathbf{X}\), the operator (spectral) norm is denoted by \(\|\mathbf{X}\|_2\), and its Frobenius norm by \(\|\mathbf{X}\|_F\). The \(k\)-th largest singular value of \(\mathbf{X}\) is denoted by \(\sigma_k(\mathbf{X})\).
For a matrix \(\mathbf{X} \in \mathbb{R}^{d_1 \times d_2}\), we use \(\mtov(\mathbf{X}) \in \mathbb{R}^{d_1 d_2}\) to denote the vector obtained by column-wise vectorizing \(\mathbf{X}\), i.e., \(\mtov(\mathbf{X})^\top = [x_1^\top, \dots, x_{d_2}^\top]\), where \(x_i\) is the \(i\)-th column of \(\mathbf{X}\). The identity matrix of size \(d \times d\) is denoted by \(\mathbf{I}_d\).

\paragraph{Markov Decision Processes.}
We consider the tabular finite-horizon Markov Decision Process (MDP) to model the Reinforcement Learning from Human Feedback (RLHF) setting with $N$ human labelers (or users), each with their own reward function. A MDP \(\Mc\) is represented by the tuple \(\Mc = (\mathcal{S}, \mathcal{A}, H, (P_h)_{h \in [H]}, \mathbf{r} = (r_{i})_{i \in [N]})\), where \(\mathcal{S}\) is the set of states, defined as all possible prompts or questions; \(\mathcal{A}\) is the set of actions, representing potential answers or responses to these questions; \(H\) denotes the length of the horizon; \(P_h: \mathcal{S} \times \mathcal{A} \to \Delta(\mathcal{S})\) is the state transition probability at step \(h \in [H]\), with \(\Delta(\mathcal{S})\) being the set of probability distributions over \(\mathcal{S}\); and \(r_i:\Tc\mapsto [-R,R]\) is the reward function for each individual \(i\in[N]\), where \(\mathcal{T} := (\mathcal{S} \times \mathcal{A})^H\) denotes the set of all possible trajectories \(\tau = (s_1, a_1, s_2, a_2, \dots, s_H, a_H)\).
The MDP concludes at an absorbing termination state with zero reward after \(H\) steps. A policy is defined as a sequence \(\pi = (\pi_h)_{h=1}^H\), where each \(\pi_h: (\mathcal{S} \times \mathcal{A})^{h-1} \times \mathcal{S} \to \Delta(\mathcal{A})\) maps the history and current state to a distribution over actions at step \(h\). The expected cumulative reward of a policy \(\pi\) for individual \(i\) is given by \( J(\pi; r_i) := \mathop{\mathbb{E}}_{\tau \sim \pi}[r_i(\tau)] \).

\paragraph{Relationship between Preference and Reward Functions.}
Given two trajectories \(\tau_0\) and \(\tau_1\), we introduce a random variable \(o \in \{0,1\}\) to represent the preference outcome: We set \(o = 1\) if \(\tau_0 \succ \tau_1\) (i.e., \(\tau_0\) is preferred over \(\tau_1\)), and \(o = 0\) if \(\tau_0 \prec \tau_1\) (i.e., \(\tau_1\) is preferred over \(\tau_0\)). We model the probability that individual \(i \in [N]\) prefers \(\tau_0\) over \(\tau_1\) as
\(P_{r_i}(o = 1 \mid \tau_0, \tau_1) = \Phi\big( r_i(\tau_0) - r_i(\tau_1) \big)\), where \(\Phi: \mathbb{R} \to [0,1]\) is a monotonically increasing function satisfying \(\Phi(x) + \Phi(-x) = 1\) and \(\log \Phi(x)\) is a Lipschitz continuous and strongly convex function. A common choice for \(\Phi\) is the sigmoid function \(\sigma(x) = 1/(1 + e^{-x})\), which maps real-valued inputs to the range \([0,1]\). This function corresponds to the Bradley-Terry-Luce (BTL) model, which is commonly used to model the relationship between preferences and rewards.
We define the \emph{preference probability vector} induced by the reward functions \(\bm{r}\) as \(P_{\bm{r}}(o \mid \tau_0, \tau_1) = \left( P_{r_1}(o \mid \tau_0, \tau_1), \dots, P_{r_N}(o \mid \tau_0, \tau_1) \right)^\top\),
where \(P_{\bm{r}}\) represents the collective preference probabilities across all individuals, and \(P_{r_i}\) denotes the preference probability induced by the reward function \(r_i\) for individual \(i\).

\paragraph{Personalized Reward Functions.}
We consider the naturally diverse individual human preferences and aim to learn personalized reward models for each individual. As a first step, we assume each reward function $r_i$ is parameterized by \(\boldsymbol{\Theta}_i \in \mathbb{R}^{d_1 \times d_2}\), and we denote it as $r_{\boldsymbol{\Theta}_i}:\mathcal{T} \to \mathbb{R}$. 
We denote the aggregated reward vector as 
\(\mathbf{r}_{\boldsymbol{\Theta}} := \big( r_{\boldsymbol{\Theta}_1}, \dots, r_{\boldsymbol{\Theta}_N} \big)^\intercal\),
where \(\boldsymbol{\Theta} \in \mathbb{R}^{d_1 \times N d_2}\) 
is the aggregated parameter matrix defined by \(\boldsymbol{\Theta} = \big[ \boldsymbol{\Theta}_1,\ \dots,\ \boldsymbol{\Theta}_N \big].\)

 Let \(\theta\) denote the column-wise vectorization of \(\Thetav\), i.e., \(\theta = \text{vec}(\Thetav)\). Then, we make the following assumption.
\begin{assumption}\label{assumption:reward}
    For any trajectory \(\tau\), the reward function \( r_{\Thetav}(\tau) \) satisfies Lipschitz continuity \(\|\nabla_{\theta} r_{\Thetav}(\tau)\| \leq L_1\) and Lipschitz smoothness \(\|\nabla^2_{\theta} r_{\Thetav}(\tau)\| \leq L_2\) for \(L_1,L_2>0\).
\end{assumption}
Note that the gradient operator \(\nabla\) and the Laplacian \(\nabla^2\) are applied to the vectorized parameter matrix \(\theta\). \Cref{assumption:reward} is a standard assumption similar to those in related RLHF studies, such as \citet{zhu2023principled}.

Define the set of valid parameters for the reward function as
\begin{align}
   \Sc:= \left\{ \Thetav \,\Big|\, \boldsymbol{\Theta}_i \in \mathbb{R}^{d_1 \times d_2},\ \|\boldsymbol{\Theta}_i\|_F \leq B,\ \forall i\in[N]\right\},\label{def:para-set}
\end{align}
and the corresponding class of reward functions as
\begin{align}
\mathcal{G}_{\mathbf{r}}(\Sc) = \left\{ \big(r_{\boldsymbol{\Theta}_i}(\cdot)\big)_{i \in [N]} \,\Big|\, \boldsymbol{\Theta}\in  \mathcal{S}\right\}.\label{def:func-class}
\end{align}
The boundedness condition $\|\boldsymbol{\Theta}_i\|_F \leq B$ ($B$ is a positive constant) in \Cref{def:para-set}, together with \Cref{assumption:reward}, ensures that the reward function is bounded, which is a standard assumption adopted in related works~\citep{zhan2023provable,zhong2024provable}.

Throughout this paper, we let 
\(\mathbf{r}^\star = \big( r_1^\star, \dots, r_N^\star \big)\)
denote the underlying true human reward functions with corresponding ground truth parameters \(\Thetav^*=[\Thetav_1^*,\cdots, \Thetav_N^\star]\). We assume that \(\mathbf{r}^\star \in \mathcal{G}_{\mathbf{r}}(\Sc)\) to ensure that the true reward functions are within the considered function class.

\paragraph{Learning Personalized Reward Functions via LoRA.}

Motivated by LoRA that is widely adopted for the fine-tuning of LLMs~\citep{sidahmed2024perl}, we assume the system starts from initialized reward model parameters \(\Thetav^{\text{init}} = [\Thetav_1^{\text{init}}, \cdots, \Thetav_N^{\text{init}}]\). Denote the low-rank adaptation matrice for the reward models as \(\Delta\Thetav = [\Delta\Thetav_1, \cdots, \Delta\Thetav_N]\). Then, after the adaptation, the set of valid parameters for the personalized reward model becomes
\begin{align}
  \Sc^{\text{LoRA}}    = &\Big\{  \Thetav \,\Big|\, \Thetav = \Thetav^{\text{init}} + \Delta\Thetav,\operatorname{rank}(\Delta\Thetav_i) \leq k,\nonumber \quad \|\Delta\Thetav_i\|_F \leq B , \forall i \in [N] \Big\}.
\end{align}

Note that the LoRA module is typically represented in a low-rank factorization form, i.e., as the product of two lower-dimensional matrices: \(\Delta \Thetav_i = \Bv_i \Wv_i\), where \(\Bv_i \in \mathbb{R}^{d_1 \times k}\) and \(\Wv_i \in \mathbb{R}^{k \times d_2}\). In the function class \(\mathcal{G}_{\mathbf{r} \mid \boldsymbol{\Theta}^{\text{init}}}\), the individual LoRA modules \(\Delta \Thetav_i\) are independent. To leverage potential common structures among the individual LoRA modules, as observed in recent works~\citep{zhu2024asymmetry,guo2024selective,tian2024hydralora}, we assume that the \(\Bv_i\) matrices are shared across all users, i.e., \(\Bv_i = \Bv\) for all \(i\). Under this constraint, the aggregated matrix \(\Delta \Thetav\) can be expressed as \(\Delta \Thetav = \Bv [\Wv_1, \cdots, \Wv_N]\), which implies that \(\Delta \Thetav\) becomes a low-rank matrix with \(\operatorname{rank}(\Delta \Thetav) \leq k\), since \(\operatorname{rank}(\Bv) \leq k\). Consequently, when \(\Bv\) is shared across all LoRA modules, the parameter set is equivalent to:
\begin{align}
  \mathcal{S}^{\text{ShareLoRA}}    = &\Big\{  \Thetav \,\Big|\, \Thetav = \Thetav^{\text{init}} + \Delta\Thetav,\operatorname{rank}(\Delta\Thetav) \leq k,\nonumber \quad \|\Delta\Thetav_i\|_F \leq B, \forall i \in [N] \Big\}.
\end{align}

To leverage the potential common structure among individual LoRA modules, we utilize the parameter set \(\Sc^{\text{ShareLoRA}}\), which allows us to learn LoRA modules with shared parameters across users effectively. This low-rank constraint leverages shared structures among users' preferences, allowing the model to capture common patterns while adapting to individual differences. The aggregated low-rank adaptation \(\Delta\Thetav\) results in local low-rank adaptations \(\{\Delta\Thetav_i\}\), which incorporate a shared matrix \(\Bv\) and distinct individual adaptation matrices \(\Wv_i\), i.e., \(\Delta\Thetav_i = \Bv\Wv_i\). Intuitively, the shared matrix \(\Bv\) preserves common directions for parameter updating, while \(\Wv_i\) captures individual adaptation along those dimensions.

Given a collection of preference datasets for individual users, denoted as \(\hat{\mathcal{D}}_i = \{ (o_i^{(j)}, \tau_{i,0}^{(j)}, \tau_{i,1}^{(j)}) \}_{j=1}^{N_p}\), our objective is to estimate the ground-truth reward function \(\mathbf{r}^\star\) by combining the learned shared-parameter LoRA matrices within $\mathcal{S}^{\text{ShareLoRA}}$. 
We define the aggregated dataset as \(\hat{\mathcal{D}} = \bigcup_{i=1}^N \hat{\mathcal{D}}_i\), with \(|\hat{\mathcal{D}}_i| = N_p\) for all \(i \in [N]\). Our analysis can be extended to scenarios where the dataset sizes vary across individuals, i.e., \(|\hat{\mathcal{D}}_i| = N_{p,i}\) for each \(i\).
The optimization problem is then formulated as follows:
\begin{equation}
\max_{\boldsymbol{\Theta}\in \mathcal{S}^{\text{ShareLoRA}} }  F\left(\boldsymbol{\Theta}; \hat{\mathcal{D}}\right) = \sum_{i=1}^N \sum_{j=1}^{N_p} \log P_{\boldsymbol{\Theta}_i} \left(o_i^{(j)} \mid \tau_{i,0}^{(j)}, \tau_{i,1}^{(j)}\right), \label{def:obj}
\end{equation}
where we use \(P_{\Thetav}\) denote \(P_{r_{\Thetav}}\) to simplify the notation.

\begin{algorithm}[ht]
	\caption{\texttt{P-ShareLoRA} for RLHF\label{alg:personal}}
	\begin{algorithmic}[1]
     \STATE \textbf{Input:} Dataset $\hat{\Dc}=\cup_{i \in [N]} \hat{\Dc}_i$;  initial parameters  \(\Thetav^{\text{init}}\); reference policy $\mu_{i, \text{ref}}$.
    \STATE Obtain model update $\Delta\hat{\Thetav}$ by solving \Cref{def:obj} :
    \begin{align*}
        \Delta\hat{\Thetav} \leftarrow \argmax_{\Delta\hat{\Thetav}:\Thetav\in \mathcal{S}^{\text{ShareLoRA}} }  F\left(\boldsymbol{\Theta}; \hat{\mathcal{D}}\right)
    \end{align*}

    \STATE Construct confidence sets \(\{\Rc_i\}_{i=1}^N\) by
        \begin{equation}
        \begin{aligned}
        \Rc_i\!\leftarrow\! 
        \biggl\{ r_{\Thetav_i} \Biggiven \Thetav_i\!=\!\Thetav^{\text{init}}_i \!+\!\Delta\Thetav_i,\|\Delta\Thetav_i\!-\!\Delta\hat{\Thetav}_i\|_F^2\!\leq\!\zeta\biggr\}
        \end{aligned}   \label{eqn:confidenceset-alg1-1}
    \end{equation}

    \STATE Compute policy with respect to \(\Rc_i\) for all $i \in [N]$ by
    \begin{align}
    \hat{\pi}_i\leftarrow \argmax_{\pi \in \Pi} \min_{r_i \in \Rc_i} \left(J(\pi; r_i) - \Eb_{\tau \sim \mu_{i, \text{ref}}}[r_i(\tau)]\right) \label{eqn:robust-alg1}
    \end{align} 
    \STATE \textbf{Output:} {$ (\Delta\hat{\Thetav}, (\hat{\pi}_i)_{i \in [N]})$}.
    \end{algorithmic}
\end{algorithm}

\section{Algorithm Design and Analysis}

\subsection{Algorithm: \texttt{P-ShareLoRA} for RLHF}
In this section, we present our proposed algorithm, Personalized LoRA with Shared Component (\texttt{P-ShareLoRA}) for RLHF, to effectively learn personalized reward functions and compute corresponding policies for each individual user.

The algorithm begins by initializing the reward function for each user \(i\) by \(\Thetav_i^{\init}\).
The core of the algorithm involves estimating the personalized reward models by optimizing low-rank adaptations \(\Delta\hat{\Thetav}\). Specifically, we obtain \(\Delta\hat{\Thetav}\) by solving the optimization problem defined in \Cref{def:obj}.

After obtaining \(\hat{\boldsymbol{\Theta}}\), we construct confidence sets \(\{\mathcal{R}_i\}\) for each user’s reward function parameters. Each set \(\mathcal{R}_i\) is designed to ensure that the distance between the parameter matrix of the reward function and the empirical estimation obtained by solving \Cref{def:obj} remains within a tolerance level \(\zeta\), thereby providing a robust confidence region for the reward functions. Finally, we compute each user's personalized policy \(\hat{\pi}_i\) by solving a robust optimization problem. For each individual \(i \in [N]\), we determine the policy that maximizes the difference between its expected cumulative reward \(J(\pi; r_i)\) and the expected reward of the reference policy \(\mu_{i,\text{ref}}\), evaluated under the worst-case reward function within the confidence set \(\mathcal{R}_i\). The algorithm outputs the estimated reward model parameters \(\hat{\boldsymbol{\Theta}}\) and the set of personalized policies \((\hat{\pi}_i)_{i \in [N]}\). We note that without pessimism (i.e., the confidence set of reward functions reduces to a singleton \(\Rc_i = \{r_{\hat{\Thetav}_i}\}\)), the optimization objective simplifies to the vanilla RLHF objective. \texttt{P-ShareLoRA} is detailed in Algorithm~\ref{alg:personal}.

\subsection{Definitions and Assumptions}
Before formally presenting our main theoretical results of \Cref{alg:personal}, we introduce the following definitions and assumptions. 
We start by defining two diversity metrics over human preference on different labelers.

\begin{definition}[Diversity Metrics]\label{def:task_diversity_sharp}
Given the aggregated ground-truth parameter matrix \( \Thetav^\star = [\Thetav^\star_1, \dots, \Thetav^\star_N] \) and initialization parameter matrices \(\{ \Thetav^{\init}_i\} \), we define the difference matrix \( \Delta\Thetav^\star = [\Delta\Thetav^\star_1, \dots, \Delta\Thetav^\star_N] \), where \( \Delta\Thetav^\star_i = \Thetav^\star_i - \Thetav^{\init}_i \) for each user \( i \). Let \( \sigma_1 \geq \sigma_2 \geq \dots \geq \sigma_{\min\{d_1, N d_2\}} \) be the singular values of \( \Delta\Thetav^\star \). We then define the condition number \( \nu \) and the summation of tail singular values \( \tail \) as
  $  \nu = \frac{\sigma_k^2}{N}, \tail = \sum_{i=k+1}^{\min\{d_1, N d_2\}} \sigma_i^2.$
\end{definition}

\begin{remark}
The condition number \(\nu\), as defined in \citep{tripuraneni2021provable}, quantifies the alignment of parameter differences between the ground truth model parameters and the initialization across users. Specifically, it considers the magnitude of the \(k\)-th largest singular value of the difference matrix \(\Delta\Thetav^\star\), normalized by the number of users \(N\). 
Note that due to the constraint in \(\Sc\), for fixed \(\Thetav^{\init}\), the bounded total energy of \(\Thetav^\star\), i.e., \(\|\Thetav^\star\|_F^2\leq NB^2\), implies the total energy of \(\Delta\Thetav^\star\) is also bounded. Therefore, a larger \(\nu\) indicates that the top-\(k\) leading singular values are significantly larger than the subsequent ones.
This dominance suggests that \(\Delta\Thetav^\star_i\) across users are primarily aligned along a few principal directions, indicating low diversity. 
Conversely, a smaller \(\nu\) indicates high diversity across different directions. 

The tail sum \(\tail\) measures the total variance not captured by the top \(k\) singular values of \(\Delta\Thetav^\star\). It is calculated by summing the squares of the singular values from \(\sigma_{k+1}\) onward, quantifying the residual ``energy'' beyond a rank-\(k\) approximation. A smaller \(\tail\) suggests that the top \(k\) singular values capture most of the variance, implying that a low-rank adaptation effectively represents the essential variability among users for accurate modeling of reward functions.

These diversity metrics capture the preference diversity among users. Intuitively, users with similar preferences will be less diverse and could benefit more from a shared LoRA model.

\end{remark}

Next, to capture the complexity of the reward function class, we introduce the concept of the bracketing number for reward vectors.

\begin{definition}[Bracketing Number for Reward Vectors~\citep{park2024rlhf}]\label{def:apx.A.1}
For a reward vector ${\rbm}\in\mathcal{G}_{\rbm}$, an $\epsilon$-bracket is a pair of functions $(g_1, g_2)$ such that for all $(\tau_0, \tau_1)\in\Tc\times\Tc$, $\| g_1(\tau_0, \tau_1) - g_2(\tau_0, \tau_1) \|_1 \leq \epsilon$, and
$g_{1}(\tau_0, \tau_1) \leq P_{\rbm}(\cdot|\tau_0, \tau_1) \leq g_{2}(\tau_0, \tau_1).$ 
The $\epsilon$-bracketing number of $\mathcal{G}_{\rbm}$, denoted by $\mathcal{N}_{\mathcal{G}_{\rbm}}(\epsilon)$, is the minimal number of $\epsilon$-brackets required to cover all ${\rbm}$ in $\mathcal{G}_{\rbm}$.
\end{definition}
\Cref{def:apx.A.1} is adapted from the definition of bracketing numbers in \citet{park2024rlhf,zhan2023provable}, which captures the complexity of the function class in terms of its parameter dimensions.

We assume a uniform concentration property for the expected Euclidean distance between \(r_{\Thetav_1}(\tau_0) - r_{\Thetav_1}(\tau_1)\) and \(r_{\Thetav_2}(\tau_0) - r_{\Thetav_2}(\tau_1)\) over the offline data. We note that this expected Euclidean distance can be seen as the distance between two reward functions \(r_{\Thetav_1}\) and \(r_{\Thetav_2}\)~\citep{zhan2023provable}, therefore the concentration property ensures that with a sufficiently large sample size \(N\), empirical data reliably approximates these distance for all pairs of reward functions in \(\Gc_r\).

\begin{assumption}[Uniform Concentration]\label{assum:2}
Given distributions \(\mu_0\) and \(\mu_1\), and two reward functions parameterized by \(\Thetav_1\) and \(\Thetav_2\), respectively, we define the expected and empirical squared difference of reward discrepancies as
\begin{align*}
    D_{\Thetav_1, \Thetav_2}(\mu_0, \mu_1) = \mathop{\mathbb{E}}_{\tau_0 \sim \mu_0,\, \tau_1 \sim \mu_1} &\big[ \big( r_{\Thetav_1}(\tau_0) - r_{\Thetav_1}(\tau_1) - \big(r_{\Thetav_2}(\tau_0) - r_{\Thetav_2}(\tau_1)\big) \big)^2 \big],\\
    \hat{D}_{\Thetav_1, \Thetav_2}(\mu_0, \mu_1)= \frac{1}{N}\sum_{\{\tau_0^j,\tau_1^j\}\in\Dc} &\big[ \big( r_{\Thetav_1}(\tau_0^j) - r_{\Thetav_1}(\tau_1^j) -\big( r_{\Thetav_2}(\tau_0^j) - r_{\Thetav_2}(\tau_1^j) \big) \big)^2 \big],
\end{align*}
where \(\Dc\) is a dataset satisfies \(|\Dc|=N\) and all trajectory pairs \(\{\tau_0^j,\tau_1^j\}\in\Dc\) are sampled from \(\mu_0\) and \(\mu_1\) respectively. Then, for any \(\delta \in (0, 1]\), there exists a number \(N_{\text{unif}}(\mathcal{G}_r, \mu_0, \mu_1, \delta)\) such that for any \(N \geq N_{\text{unif}}(\mathcal{G}_r, \mu_0, \mu_1, \delta)\), the empirical estimate \(\hat{D}_{\Thetav_1, \Thetav_2}(\mu_0, \mu_1)\) of \(D_{\Thetav_1, \Thetav_2}(\mu_0, \mu_1)\) satisfies the following inequality with probability at least \(1 - \delta\) for all \(r_{\Thetav_1}, r_{\Thetav_2} \in \mathcal{G}_r\):
$0.9\, D_{\Thetav_1, \Thetav_2}(\mu_0, \mu_1) \leq \hat{D}_{\Thetav_1, \Thetav_2}(\mu_0, \mu_1) \leq 1.1\, D_{\Thetav_1, \Thetav_2}(\mu_0, \mu_1)$.
\end{assumption}
\Cref{assum:2} indicates that the empirical estimate \(\hat{D}_{\Thetav_1, \Thetav_2}(\mu_0, \mu_1)\) closely approximates the true value \(D_{\Thetav_1, \Thetav_2}(\mu_0, \mu_1)\) with high probability.  This assumption is crucial in our context because it ensures that, given a sufficiently large sample size \(N\), the empirical data provides a reliable approximation of the expected squared differences in reward discrepancies across all pairs of reward functions in \(\Gc_r\). A similar assumption is adopted by \citet{zhan2023provable} and proved to be held when the reward function is constructed by a linear representation and linear local head \citep{zhong2024provable}. 
We note that this assumption is analogous to the uniform concentration results commonly used in statistical learning, where empirical estimates converge uniformly to their expected values over a class of functions (see, e.g., \citet{vershynin2018high,du2020few,tripuraneni2021provable}). It is a mild assumption and can be satisfied for various function classes. For example, polynomial functions of bounded degrees satisfy this assumption.

\subsection{Main Results} \label{sec:results}

Building upon the aforementioned definitions and assumptions, we now present our main theoretical results. 
For ease of exposition, we denote \(\Gc_{\rbm}(\Sc^{\text{ShareLoRA}})\) by \(\Gc'_{\rbm}\).

First, we demonstrate that the column space of \(\Delta\hat{\Thetav}\), obtained via \Cref{alg:personal}, closely approximates the optimal rank-\(k\) representation of \(\Delta\Thetav^\star\). For the low-rank matrix \(\Delta\hat{\Thetav}\), let its SVD be \(\Delta\hat{\Thetav} = \hat{\Bv} \hat{\Sigmav} \hat{\Vv}^\top\). Consequently, the column space of \(\Delta\hat{\Thetav}\) is spanned by the orthonormal matrix \(\hat{\Bv}\), i.e., \(\spn\{\Delta\hat{\Thetav}\} = \spn\{\hat{\Bv}\}\). 

For \(\Delta\Thetav^\star\), we define its optimal rank-\(k\) approximation as
\begin{align}
    \Thetav^\diamond = \argmin_{\Delta\Thetav:\,\rank(\Delta\Thetav)=k} \|\Delta\Thetav^\star - \Delta\Thetav\|_F.\label{def:5.1}
\end{align}
Existing results in low-rank matrix factorization \citep{golub2013matrix} indicate that the solution must satisfy \(\Thetav^\diamond = \Uv_k \Lambdav_k \Vv_k^\top\), where \(\Lambdav_k\) is a \(k \times k\) diagonal matrix containing the top-\(k\) singular values of \(\Delta\Thetav^\star\), and \(\Uv_k\) and \(\Vv_k\) are the corresponding left and right singular vectors, respectively.
Let \(\Bv^\diamond = \Uv_k\) and \(\Wv^\diamond = \Lambdav_k \Vv_k^\top\), which yields \(\Thetav^\diamond = \Bv^\diamond \Wv^\diamond\). Therefore, the column space of the optimal rank-\(k\) estimation of \(\Delta\Thetav^\star\) is given by \(\Bv^\diamond\), and the corresponding LoRA module for each individual reward function can be expressed as: \(\Delta\Thetav_i = \Bv \Wv_i^\diamond\) for all \(i\in[N]\), where \(\Wv^\diamond = [\Wv^\diamond_1\cdots\Wv_N^\diamond].\)

To quantify the closeness between the subspaces spanned by \(\hat{\Bv}\) and \(\Bv^\diamond\), we employ the principal angle distance, as detailed in \Cref{sec:deferred-def}. Utilizing this metric, we establish the following theorem.

\begin{restatable}{theorem}{corsft}
    \label{cor:5.2}
    \emph{(Closeness between \(\hat{\Bv}\) and \(\Bv^\diamond\)).}
    Suppose \Cref{assumption:reward} holds.
    For any \(\delta \in (0,1]\), with probability at least \(1 - \delta\), it holds that
    \begin{align*}
        &\dist(\hat{\Bv}, \Bv^\diamond)\quad\leq c_1\sqrt{
        \frac{1}{NN_p\nu} \log\left( \mathcal{N}_{\mathcal{G}'_{\mathbf{r}}}\left( \frac{1}{N N_p} \right) \frac{1}{\delta} \right)+\frac{1}{\nu}\sqrt{\frac{\tail}{N}}},
    \end{align*}
where \(c_1 > 0\) is a constant.
\end{restatable}

The detailed proof is deferred to \Cref{apx:D}. 

\begin{remark}
    In \Cref{cor:5.2}, we demonstrate that the principal angle distance between \(\hat{\Bv}\) and \(\Bv^\diamond\) decreases as the condition number increases. This implies that when the \(k\)-th singular value approaches the maximum singular value of \(\Delta\Thetav^\star\), which is upper bounded by a constant due to the assumption in \Cref{def:func-class} that \(\|\Delta\Thetav^\star_i\|_F\) is bounded, the principal angle distance diminishes. This suggests that greater similarity among human users contributes to a more accurate estimate \(\hat{\Bv}\).
    
    Furthermore, the bias term in \Cref{cor:5.2}, given by \(\frac{1}{\nu}\sqrt{\frac{\tail}{N}}\), decreases as the condition number increases and as the sum of the tail singular values decreases. Specifically, the bias term vanishes when all tail components are zero, meaning it disappears if there exists a ground-truth low-rank representation \(\Bv^\star\) such that \(\Delta\Thetav_i^\star = \Bv^\star \Wv^\star\) for all \(i \in [N]\).
\end{remark}

In \Cref{cor:5.2}, the principal angle distance is also influenced by the bracketing number \(\mathcal{N}_{\mathcal{G}'_{\mathbf{r}}}\). We establish an upper bound on this quantity in the following proposition:

\begin{restatable}{proposition}{cond}
\label{prop:1}
Suppose \Cref{assumption:reward} holds. Then, the bracketing number for function class \(\Nc_{\Gc'_{\rbm}}\) satisfies 
\begin{align}
    \log \left( \mathcal{N}_{\mathcal{G}'_{\mathbf{r}}} ( (NN_p)^{-1} )/\delta \right) 
    \leq \Oc \big( k (d_1 + N d_2) \log (N N_p / \delta) \big).\label{ineq:4.4}
\end{align}
\end{restatable}

The proof is deferred to \Cref{apx:D}. We observe that the reward function class \(\mathcal{G}_{\mathbf{r}}(\Sc)\), as defined in \Cref{def:func-class}, has a bracketing number satisfying 
\[\log\left( \mathcal{N}_{\mathcal{G}_{\mathbf{r}}}( (NN_p)^{-1} )/\delta \right)\leq\Oc \big( Nd_1d_2 \log (N N_p / \delta) \big)\]
This result indicates that the bound for \(\mathcal{G}'_{\mathbf{r}}\) is significantly improved compared with full-parameter fine-tuning when \(d_1 \gg k\). 

Besides, when each LoRA module is learned individually (i.e., \(\Thetav\in\Sc^{\text{LoRA}}\)), the reward function class \(\mathcal{G}_{\mathbf{r}}(\Sc^{\text{LoRA}})\) satisfies
\[\log\left( \mathcal{N}_{\mathcal{G}_{\mathbf{r}}}( (NN_p)^{-1} )/\delta \right)\leq\Oc \big( Nk(d_1 + d_2) \log (N N_p / \delta) \big)\]
Compared to~\Cref{ineq:4.4}, our shared-component LoRA method reduces the bracketing number by decreasing the term from \(Nd_1k\) to \(d_1k\).

Next, we establish a bound on the gap in expected value functions between the target policy \(\pi_{i,\text{tar}}\) and the estimated policy \(\hat{\pi}_i\) for each individual \(i \in [N]\). In this context, \(\pi_{i,\text{tar}}\) serves as a benchmark for evaluating the performance of \(\hat{\pi}_i\); for instance, it may represent the optimal policy \(\pi_i^\star\) associated with the true reward function \(r_i^\star\).

\begin{restatable}{theorem}{worstcasesft}
\label{thm:worst-case-sft}
 \emph{(Individual Expected Value Function Gap).} 
 Suppose \Cref{assumption:reward} and \Cref{assum:2} hold.
 For any $\delta \in (0, 1]$, with probability {at least} $1-\delta$, the output $\hat{\pi}_i$ for any client \(i\) satisfies  
 \begin{align*}
    J(\pi_{i, \text{tar}}; r^\star_i) - J(\hat{\pi}_i; r^\star_i)\leq 
         c_2 \sqrt{\left(
        \frac{\log\left( \mathcal{N}_{\mathcal{G}'_{\rbm}}(\frac{1}{NN_p})\frac{1}{\delta}\right)}{N N_p \nu} +
        \frac{ k d_2 + \log(\frac{N}{\delta}) }{ N_p }+\frac{1}{\nu} \sqrt{ \frac{\tail}{ N } } +b_i
    \right)}
  \end{align*}
% \resizebox{\linewidth}{!}{
%   \begin{minipage}{\linewidth}
%   \begin{align*}
%     J(\pi_{i, \text{tar}}; r^\star_i) - J(\hat{\pi}_i; r^\star_i)\leq 
%          c_2 \sqrt{\left(
%         \frac{\log\left( \mathcal{N}_{\mathcal{G}'_{\rbm}}(\frac{1}{NN_p})\frac{1}{\delta}\right)}{N N_p \nu} +
%         \frac{ k d_2 + \log(\frac{N}{\delta}) }{ N_p }+\frac{1}{\nu} \sqrt{ \frac{\tail}{ N } } +b_i
%     \right)}
%   \end{align*}
%   \end{minipage}
% }
where \(b_i\) is defined as \(b_i:=\left\| \Delta\Thetav_i^\star - \Thetav_i^\diamond \right\|_F^2\) and $c_2 >0$ is a constant. 
\end{restatable}
\textbf{Proof Sketch.}
We face two core challenges in our analysis. First, the reward functions are inferred from preference data rather than observed directly, introducing estimation noise that must be carefully controlled. Second, due to the low-rank structure imposed on the LoRA modules, the globally optimal shared LoRA may not perfectly capture the ground-truth reward parameters for each local dataset. This misalignment complicates the analysis of how well a single shared solution performs across different local tasks.

To address the first challenge, we leverage the continuity to translate small deviations in preference space into bounded deviations in parameter space. For the second challenge, we develop a Lagrange remainder-based analysis that quantifies the approximation error introduced by the low-rank constraint. Although perfect recovery is not guaranteed, we show that the resulting estimation error remains bounded.

The proof consists of three major steps: (1) Upper bound the distance between the column space between \(\Delta\hat{\Thetav}\) and \(\Delta{\Thetav}^\star\) (\Cref{cor:label-is-correct}); (2) Analyze the distance between the learned reward function from \cref{alg:personal}\(\hat{r}_i\) and ground truth reward function \(r^\star_i\) (\Cref{thm:step2}); (3) Showing the value function of the learn policy is close to the reference policy (\Cref{thm:diverse-newmodel}).

In Step 1, we utilize the existing result of MLE estimates over the preference dataset, upper bound the distance between the estimated share component LoRA matrix with the ground truth parameter matrix, and then use the Davis-Kahan theorem to bound the corresponding distance between the column space of these two matrices.

In Step 2, for learned reward function with parameter matrix \(\hat{\Thetav}_i=\hat{\Bv}\hat{\Wv}_i\) and optimal low-rank approximated reward function parameterized by \(\Thetav^\diamond_i=\Bv^\diamond \Wv^\diamond_i\), we decompose the distance between the two functions into two part: distance between \(\hat{\Bv}\) and \(\Bv^\diamond\), which already bounded in Step 1, and the distance between \(\hat{\Wv}_i\) and \(\Wv^\diamond_i\). For this distance, we carefully analyze the geometry of the reward function around the local optimal and utilize the Lagrange remainder to construct a delicate quadratic form of the gradient for \(\Wv\), therefore upper bound the distance between \(\hat{\Wv}_i \) and \(\Wv^\diamond_i\).

In Step 3, we use the result from Step 2 along with \Cref{assum:2} to show that the expected Euclidean distance between \(r_{\Thetav_1}(\tau_0) - r_{\Thetav_1}(\tau_1)\) and \(r_{\Thetav_2}(\tau_0) - r_{\Thetav_2}(\tau_1)\) is small. Applying the pessimism mechanism from \Cref{alg:personal}, we then demonstrate that the difference between the value function of the learned policy and that of the reference policy is upper bounded by the Euclidean distance between reward functions.

A natural extension of the individual expected value function gap is the averaged bound, which provides insights into the general performance across all clients.
\begin{restatable}{corollary}{averagesft}
\label{thm:average-sft}
 \emph{(Averaged Expected Value Function Gap).} 
 Suppose \Cref{assumption:reward} and \Cref{assum:2} hold.
 For any $\delta \in (0, 1]$, with probability {at least} $1-\delta$, the output policies $\{\hat{\pi}_i\}_{i=1}^N$ satisfy 
 \begin{align*}
        \frac{1}{N}\sum_{i \in [N]} \left(J(\pi_{i, \text{tar}}; r^\star_i) - J(\hat{\pi}_i; r^\star_i)\right)\leq c_3\sqrt{\left(
        \frac{\log\left( \mathcal{N}_{\mathcal{G}'_{\rbm}}(\frac{1}{NN_p})\frac{1}{\delta}\right)}{N N_p \nu} +
        \frac{ k d_2 + \log(\frac{N}{\delta}) }{ N_p }+\frac{1}{\nu} \sqrt{ \frac{\tail}{ N } }\right)},
\end{align*}  
where $c_3 >0$ is a constant. 
\end{restatable}

\begin{remark}[Sample Complexity]
For full-parameter fine-tuning, the sample complexity required to ensure that the averaged expected value function gap is less than \(\epsilon\) with probability at least \(1 - \delta\) is 
\(N_p = \mathcal{O}\left( \frac{d_1 d_2}{\epsilon} \log\left( \frac{N}{\delta} \right) \right)\)~\citep{zhu2023principled}. 
In contrast, when using \Cref{alg:personal}, the sample complexity required to achieve an averaged estimated value function accuracy of \(1 - \epsilon - \left(\frac{\tail}{N}\right)^{1/4}\) is
\[N_p = \mathcal{O}\left( \frac{d_1 k + N d_2 k}{N \epsilon} \log\left( \frac{N}{\delta} \right) \right).\]
Therefore, when \(d_1 \gg k\), the sample complexity is significantly reduced, with the trade-off being introducing a bias term in the estimation accuracy of the value function.

Moreover, \citet{park2024rlhf} indicate that their representation learning-based method can learn an \(\epsilon\)-optimal policy with a sample complexity of
\[N_p = \mathcal{O}\left( \frac{d_1 k + N k}{N \epsilon} \log\left( \frac{N}{\delta} \right) \right).\]
Notably, in their setting, \(d_2\) is assumed to be \(1\), and the ground truth reward functions are posited to share a common representation with linear heads. In contrast, our results demonstrate a similar sample complexity with an additional bias term \(\left(\frac{\tail}{N}\right)^{1/4}\) in the accuracy. Importantly, this bias term vanishes if a ground-truth low-rank representation \(\Bv^\star\) exists such that \(\Thetav_i^\star = \Bv^\star \Wv^\star\) for all \(i \in [N]\). Hence, we can achieve similar sample complexity but for the more general reward function class and without assuming the existence of ground truth common representation.

\end{remark}

\section{Experimental Results}

\textbf{Models and Datasets.}
We implement the baseline algorithms \texttt{Share Rep}, \texttt{LoRA-local}, and \texttt{LoRA-global}, which will be introduced later,
alongside our proposed algorithms on two models: GPT-J 6B~\citep{gpt-j} and Llama-3 8B~\citep{touvron2023llama}. This setup enables a comparison with the work of \citet{park2024rlhf}.
Implementation details for all algorithms are provided in \Cref{apx:impl}, and the code is publicly available\footnote{\href{https://github.com/DonghaoLee/Shared-LoRA-Reward}{https://github.com/DonghaoLee/Shared-LoRA-Reward}}.

We empirically evaluate our algorithms on the text summarization task using the Reddit TL;DR summarization and human feedback dataset~\citep{stiennon2020learning}.
This dataset contains a broad range of user preferences, which provides a particularly suitable setting for studying personalized feedback and allows us to validate the proposed P-ShareLoRA method for learning individualized reward functions. Following \citet{park2024rlhf}, we rank the labelers by the number of annotated comparisons in the training split and select the top five workers. To balance the dataset, we cap each worker's samples to match the worker with the fewest comparisons, resulting in 5,373 samples per worker and 26,865 training samples in total. The same process is applied to the validation set, yielding 1,238 samples per worker and 6,190 validation samples overall.

\textbf{Baselines.} 
To evaluate our approach, we introduce two naive baselines for comparison: \texttt{LoRA-Global}, in which we train one shared LoRA module across all users; and \texttt{LoRA-Local}, where for each labeler's preference dataset, we independently train a separate LoRA module, allowing each user's model to fully adapt to their specific preferences without leveraging shared information across users.

To practically solve \Cref{def:obj}, we propose three alternative algorithms to obtain personalized LoRA modules with shared components: \texttt{P-ShareLoRA(SI)}, \texttt{P-ShareLoRA(G)} and \texttt{P-ShareLoRA(WU)}.

The first algorithm, \texttt{P-ShareLoRA(SI)}, where SI denotes Standard Initialization, initializes the shared \(B\) matrix to zero for all users, while each personalized matrix \(A_i\) is initialized with samples from a normal distribution. Both the shared \(B\) and the personalized \(A_i\) matrices are updated through optimizing the objective function outlined in \Cref{def:obj} using the adamW~\citep{loshchilov2017decoupled} optimizer.

The second algorithm, \texttt{P-ShareLoRA(G)}, where G denotes Global, initializes the model by pre-training the LoRA module on the entire user dataset, using the configuration from \texttt{LoRA-Global}. Training then proceeds in the same manner as in \texttt{P-ShareLoRA(SI)}. 
The third algorithm, \texttt{P-ShareLoRA(WU)}, where WU denotes warm-up, employs a few preliminary warm-up steps using a global adaptation module (similar to \texttt{P-ShareLoRA(G)}) before proceeding with user-specific training. Following this phase, training continues as in \texttt{P-ShareLoRA(SI)}. Detailed pseudocode and parameter settings for each of these algorithms are provided in \Cref{apx:alg}.

We additionally include the shared representation method by \citet{park2024rlhf} as another baseline, abbreviated as “Share Rep” in \Cref{fig:fig1}. In this algorithm, the first 70\% of the reward model’s layers are frozen as the shared representation, while the remaining 30\% are treated as personalized heads.

\begin{figure}[t]
\centering
\includegraphics[width=0.8\textwidth]{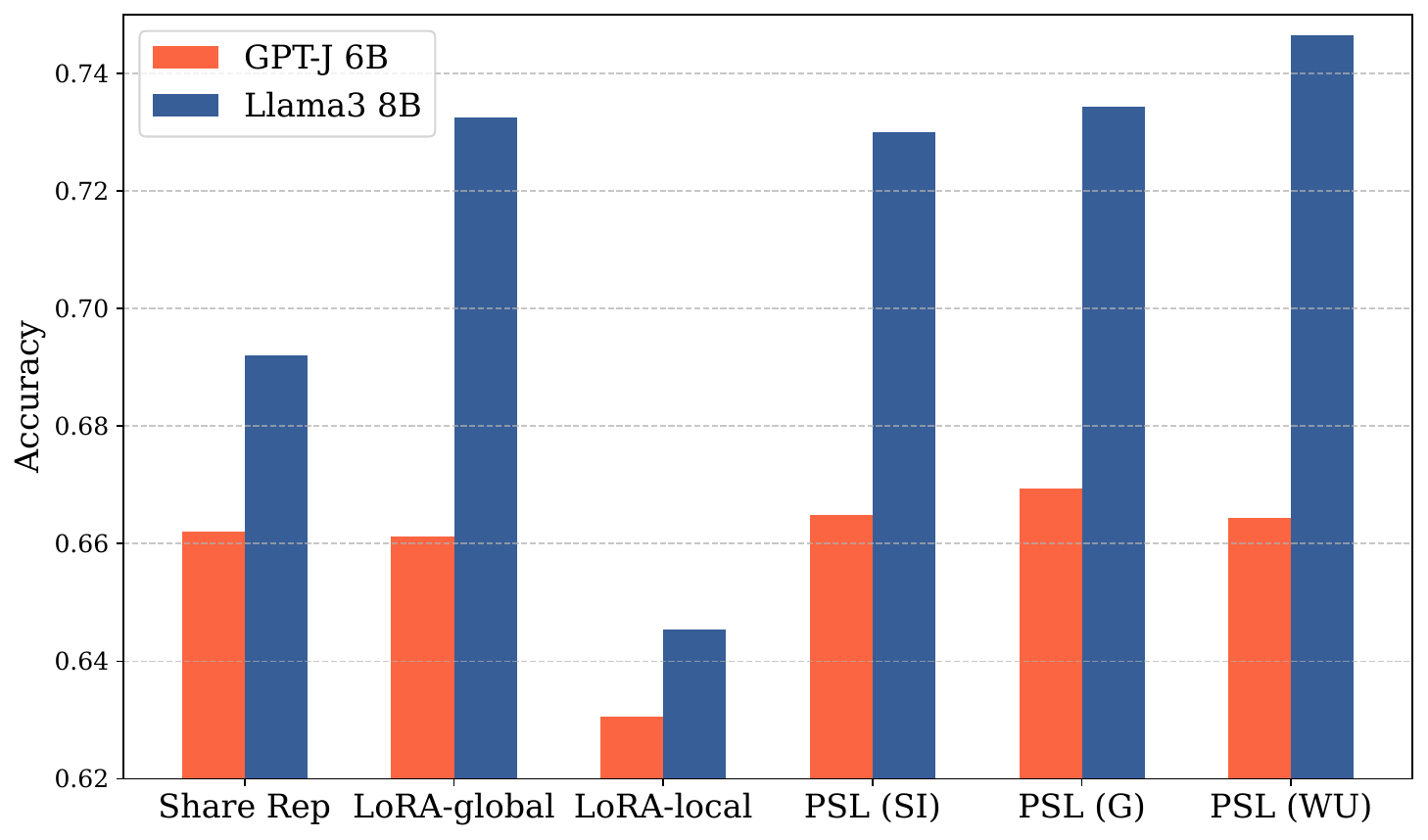}
\vspace{-0.1in}
\caption{Prediction Accuracy of Different Algorithms.}
\vspace{-0.2in}
\label{fig:fig1}
\end{figure}
% \newpage
\textbf{Results.}\label{sec:result}
For each method, we train it for a total of \(3\) epochs. Specifically, the global pretraining phase in \texttt{P-ShareLoRA(G)} is set to two epochs, while in \texttt{P-ShareLoRA(WU)} it is set to \(0.3\) epochs. Following these warm-up phases, we train \texttt{P-ShareLoRA(G)} and \texttt{P-ShareLoRA(WU)} for one and \(2.7\) epochs, respectively, ensuring that the total number of training steps remains uniform across all algorithms.

In \Cref{fig:fig1}, we present the results of reward model fine-tuning using different algorithms. The reported accuracy represents the average accuracy across the test datasets of the five labelers when preferences are estimated using each algorithm. The abbreviation PSL represents \texttt{P-ShareLoRA}.

We observe that for both GPT-J 6B and Llama-3 8B models, our proposed algorithms \texttt{P-ShareLoRA(G)} and \texttt{P-ShareLoRA(WU)} demonstrate performance improvements over other baseline algorithms. Specifically, \texttt{P-ShareLoRA(G)} achieves the most significant enhancement on GPT-J 6B, while \texttt{P-ShareLoRA(WU)} performs best on Llama-3 8B. These empirical results validate the effectiveness of our method, which leverages the shared components of LoRA modules to adapt personalized reward functions. Additional experimental results are presented in \Cref{apx:D.3}.

\section{Conclusion}

In this work, we introduced a novel algorithm that integrates LoRA into the personalized RLHF framework to effectively align LLMs with diverse user preferences. By applying LoRA to an aggregated parameter matrix, our method captures individual user preferences while leveraging shared structures, thereby improving the sample complexity and enjoying the computational efficiency of LoRA. Theoretical analysis demonstrates that \texttt{P-ShareLoRA} results in a low-rank approximation for the ground truth aggregated parameter matrix and achieves near-optimal policy performance, with performance discrepancies controlled by the diversity of user preferences. Empirical evaluations on the Reddit TL;DR dataset exhibit performance improvements compared to baseline algorithms.

\section*{Acknowledgement}
The work of R. Liu, D. Li and J. Yang was supported in part by the U.S. National Science Foundation under the grants ECCS-2133170 and ECCS-2318759. The work of P. Wang and C. Shen was supported in part by the U.S. National Science Foundation under the grants CNS-2002902, ECCS-2029978, ECCS-2143559, ECCS-2033671, CPS-2313110, and ECCS-2332060. 

\bibliographystyle{apalike}
\bibliography{refs,cpbib}

\clearpage
\onecolumn
\appendix
\centerline{{\fontsize{18}{18}\selectfont \textbf{Supplementary Materials}}}
\tableofcontents

\newpage

\section{Deferred Definitions and Preliminary Lemmas}\label{sec:deferred-def}

In our proof, we assume that all reward models are initialized from the same initial parameter matrix, i.e., \(\Thetav_i^{0} = \Thetav^{\init}\) for any \(i \in [N]\). We note that our results can be straightforwardly generalized to the case with heterogeneous initialization. Additionally, we use \(\mathbf{X}^{(N)}\) to represent the column-wise replication of matrix \(\mathbf{X}\) \(N\) times, i.e., \(\mathbf{X}^{(N)} = [\mathbf{X}, \dots, \mathbf{X}]\).

\subsection{Deferred Definitions}
Also, we introduce the following deferred definitions:

\begin{definition}[Principal Angle Distance \citep{jain2013low}]
    Given \(\Bv_1, \Bv_2 \in \mathbb{R}^{d \cdot k}\) with orthonormal columns, the principal angle distance between their column spaces is defined as
    {
    \begin{align*}
        \dist(\Bv_1, \Bv_2) = \frac{1}{\sqrt{2}} \| \Bv_1 \Bv_1^\top - \Bv_2 \Bv_2^\top \|_F = \| \Bv_1^\top \bar{\Bv}_2 \|_F,
    \end{align*}}
    where \(\bar{\Bv}_2\) is an orthonormal basis for the orthogonal complement of \(\mathrm{span}(\Bv_2)\), i.e., \(\mathrm{span}(\bar{\Bv}_2) = \mathrm{span}(\Bv_2)^\perp\).
\end{definition}

The principal angle distance is a standard metric for measuring the distance between subspaces \citep{jain2013low, collins2021exploiting}.

\begin{definition}[Bracketing Number for Single Reward~\citep{zhan2023provable}]
Consider the class $\mathcal{G}_{r}$ of functions mapping pairs of trajectories $(\tau_0, \tau_1) \in \mathcal{T} \cdot \mathcal{T}$ to preference probability vector. Specifically, each function $r \in \mathcal{G}_{r}$ maps $(\tau_0, \tau_1)$ to $P_{r}(\cdot \mid \tau_0, \tau_1) \in \mathbb{R}^{2}$. An $\epsilon$-bracket for $\mathcal{G}_{r}$ is a pair of functions $(g_1, g_2)$ mapping $\mathcal{T} \cdot \mathcal{T}$ to $\mathbb{R}^{2}$ such that for all $(\tau_0, \tau_1) \in \mathcal{T} \cdot \mathcal{T}$: (1). \(g_1(\tau_0, \tau_1) \leq g_2(\tau_0, \tau_1)\); (2). \(\| g_1(\tau_0, \tau_1) - g_2(\tau_0, \tau_1) \|_1 \leq \epsilon\)
The $\epsilon$-bracketing number of $\mathcal{G}_{r}$, denoted by $\mathcal{N}_{\mathcal{G}_{r}}(\epsilon)$, is the minimal number of $\epsilon$-brackets required to cover $\mathcal{G}_{r}$ in the following sense: for any function $r \in \mathcal{G}_{\rbm}$, there exists an $\epsilon$-bracket $(g_{b,1}, g_{b,2})$ such that for all $(\tau_0, \tau_1) \in \mathcal{T} \cdot \mathcal{T}$,
\[
    g_{b,1}(\tau_0, \tau_1) \leq P_{r}(\cdot \mid \tau_0, \tau_1) \leq g_{b,2}(\tau_0, \tau_1).
\]
\end{definition}

\begin{definition}[Concentrability Coefficient~\citep{zhan2023provable}]
\label{def:concentrability-coef}
Given a reward vector class \(\mathcal{G}_{\rbm}\), a human user \(i\), a target policy \(\pi_{\mathrm{tar}}\) (which could potentially be the optimal policy \(\pi_i^\star\) corresponding to the true reward \(r^\star_i\)), and a reference policy \(\mu_{\mathrm{ref}}\), the \emph{concentrability coefficient} is defined as:
\[
C_{\rbm}\left( \mathcal{G}_{\rbm}, \pi_{\mathrm{tar}}, \mu_{\mathrm{ref}}, i \right) := \max \left\{ 0,\ \sup_{\rbm \in \mathcal{G}_{\rbm}} \frac{ \mathbb{E}_{\tau_0 \sim \pi_{\mathrm{tar}},\ \tau_1 \sim \mu_{\mathrm{ref}}} \left[ r^\star_i(\tau_0) - r^\star_i(\tau_1) - r_i(\tau_0) + r_i(\tau_1) \right] }{ \sqrt{ \mathbb{E}_{\tau_0,\, \tau_1 \sim \mu_{\mathrm{ref}}} \left[ \left( r^\star_i(\tau_0) - r^\star_i(\tau_1) - r_i(\tau_0) + r_i(\tau_1) \right)^2 \right] } } \right\}.
\]
\end{definition}

\subsection{Preliminary Lemmas}

Before presenting the proof, we introduce a few important lemmas.

\begin{lemma}[(\citet{zhan2023provable}, Lemma 1, reward vector version)]
\label{lemma:mle}
For any $\delta \in (0,1]$, if $\rbm \in \Gc_{\rbm}$, with dataset $\hat{\Dc} = \cup_{i \in [N]} \hat{\Dc}_i$ where $\hat{\Dc}_i = \{(o_i^{(j)}, \tau_{i, 0}^{(j)}, \tau_{i,1}^{(j)})_{j \in [N_{p}]}\}$, $\tau_{i, 0}^{(j)} \sim \mu_0$, $\tau_{i, 1}^{(j)} \sim \mu_1$, and $o_i^{(j)} \sim P_{r^\star_i}(\cdot|\tau_0^{(j)}, \tau_1^{(j)})$, there exist $C_1 > 0$ such that 
\begin{align*}
    \sum_{i \in [N]} \sum_{j \in [N_{p}]} \log \left(\frac{P_{r_i}( o_{i}^{(j)} \mid \tau_{i, 0}^{(j)}, \tau_{i,1}^{(j)})}{P_{r^\star_i}(o_{i}^{(j)} \mid \tau_{i, 0}^{(j)}, \tau_{i,1}^{(j)})} \right) \leq C_1 \log(\Nc_{\Gc_{\rbm}}(1/(NN_p))/ \delta)
\end{align*} 
holds.
\end{lemma}

\begin{lemma}[(\citet{liu2022partially}, Proposition 14, scalar version)]
For any $\delta \in (0,1]$, with probability at least $1-\delta$, if $r \in \Gc_{r}'$, with dataset $\hat{\Dc} =  \{(o^{(j)}, \tau_{ 0}^{(j)}, \tau_{1}^{(j)})_{j \in [M]}\}$ where $\tau_0^{(j)} \sim \mu_0$, $\tau_1^{(j)} \sim \mu_1$, and $o^{(j)} \sim P_{r^\star}(\cdot|\tau_0^{(j)}, \tau_1^{(j)})$,  
\label{lemma:l2distance-scalar}
    \begin{align*}
        &\Eb_{\mu_0, \mu_1}\left[ \norm{P_{r} ( \cdot \mid \tau_{ 0}^{(j)}, \tau_{1}^{(j)}) - P_{r^\star} ( \cdot \mid \tau_{0}^{(j)}, \tau_{1}^{(j)})}_1^2 \right] \\
        &\leq \frac{C_2}{M} \left(    \sum_{j \in [M]} \log \left(\frac{P_{r^\star}( o^{(j)} \mid \tau_{ 0}^{(j)}, \tau_{1}^{(j)})}{P_{r}( o^{(j)} \mid \tau_{ 0}^{(j)}, \tau_{1}^{(j)})}\right) + \log(\Nc_{\Gc_{r}'}(1/M)/ \delta)\right)
    \end{align*}
holds where $C_2>0$ is a constant.
\end{lemma}

\begin{lemma}[(\citet{liu2022partially}, Proposition 14, vector version)]
For any $\delta \in (0,1]$, with probability at least $1-\delta$, if $\rbm \in \Gc_{\rbm}'$, with dataset $\hat{\Dc} = \cup_{i \in [N]} \hat{\Dc}_i$ where $\hat{\Dc}_i = \{(o_i^{(j)}, \tau_{i, 0}^{(j)}, \tau_{i,1}^{(j)})_{j \in [N_{p}]}\}$, $\tau_{i, 0}^{(j)} \sim \mu_0$, $\tau_{i, 1}^{(j)} \sim \mu_1$, and $o_i^{(j)} \sim P_{r^\star_i}(\cdot|\tau_0^{(j)}, \tau_1^{(j)})$, 
\label{lemma:l2distance}
    \begin{align*}
        &\frac{1}{N} \sum_{i \in [N]} \Eb_{\mu_0, \mu_1}\left[ \norm{P_{r_i} ( \cdot \mid \tau_{ 0}^{(j)}, \tau_{1}^{(j)}) - P_{r^\star_i} ( \cdot \mid \tau_{0}^{(j)}, \tau_{1}^{(j)})}_1^2 \right] 
        \\
        &\qquad \leq \frac{C_2}{NN_p} \left(   \sum_{i \in [N]} \sum_{j \in [N_p]} \log \left(\frac{P_{r^\star_i}( o^{(j)} \mid \tau_{ 0}^{(j)}, \tau_{1}^{(j)})}{P_{r_i}( o^{(j)} \mid \tau_{ 0}^{(j)}, \tau_{1}^{(j)})}\right) + \log(\Nc_{\Gc_{\rbm}'}(1/(NN_p))/ \delta)\right)
    \end{align*}
holds where $C_2 > 0$ is a constant. 
\end{lemma}
\begin{lemma}\label{lemma:4}
    For any use \(i\in[N]\), we have the following inequality holds:
    \begin{align*}
    \frac{1}{N} \sum_{i \in [N]} \left|\log\Phi(r_i^\star(\tau_0) - r_i^\star(\tau_1)) - \log\Phi(r_i^\diamond(\tau_0) - r_i^\diamond(\tau_1))\right|
    \leq 2LL_1 \sqrt{\frac{\tail}{N}}.
\end{align*}
\end{lemma}
\begin{proof}
 From the \(L\)-Lipschitz continuity of the function \(\log\Phi(x)\), for any trajectories \(\tau_0\) and \(\tau_1\), we have
\begin{align*}
    \left|\log\Phi(r_i^\star(\tau_0) - r_i^\star(\tau_1)) - \log\Phi(r_i^\diamond(\tau_0) - r_i^\diamond(\tau_1))\right|
    \leq L \left|r_i^\star(\tau_0) - r_i^\star(\tau_1) - r_i^\diamond(\tau_0) + r_i^\diamond(\tau_1)\right|.
\end{align*}
Recalling that \(r_i^\star(\tau) = r(\tau;\Thetav_i^\star)\) and \(r_i^\diamond(\tau) = r(\tau;\Thetav_i^\diamond)\), from the \(L_1\)-Lipschitz continuity of the function \(r(\tau;\Thetav)\) with respect to \(\Thetav\), we have
\begin{align*}
    \left|r_i^\star(\tau_0) - r_i^\star(\tau_1) - r_i^\diamond(\tau_0) + r_i^\diamond(\tau_1)\right|
    \leq 2L' \|\Thetav_i^\star - \Thetav_i^\diamond\|_F.
\end{align*}
Therefore,
\begin{align*}
    &\frac{1}{N}\sum_{i \in [N]} \left|r_i^\star(\tau_0) - r_i^\star(\tau_1) - r_i^\diamond(\tau_0) + r_i^\diamond(\tau_1)\right|\\
    &\leq \sqrt{\frac{1}{N} \sum_{i \in [N]} \left(r_i^\star(\tau_0) - r_i^\star(\tau_1) - r_i^\diamond(\tau_0) + r_i^\diamond(\tau_1)\right)^2} \\
    &\leq 2L_1 \sqrt{\frac{1}{N} \|\Thetav^\star - \Thetav^\diamond\|_F^2}.
\end{align*}
Note that \(\Thetav^\diamond\) can be derived from the truncated singular value decomposition (SVD) of \(\Thetav^\star\), retaining only the top \(k\) singular values and their corresponding singular vectors \citep{golub2013matrix,liu2024federated}. Consequently, we have
\begin{align*}
    \frac{1}{N} \sum_{i \in [N]} \left|r_i^\star(\tau_0) - r_i^\star(\tau_1) - r_i^\diamond(\tau_0) + r_i^\diamond(\tau_1)\right|
    \leq 2L_1 \sqrt{\frac{\tail}{N}}.
\end{align*}
Therefore, we finally have
\begin{align*}
    \frac{1}{N} \sum_{i \in [N]} \left|\log\Phi(r_i^\star(\tau_0) - r_i^\star(\tau_1)) - \log\Phi(r_i^\diamond(\tau_0) - r_i^\diamond(\tau_1))\right|
    \leq 2LL_1 \sqrt{\frac{\tail}{N}}.
\end{align*}
    
\end{proof}

\begin{lemma}\label{lemma:6}
For local reward models parameterized by \(\Thetav_1, \dots, \Thetav_N\), suppose there exists a constant \(\delta > 0\) such that
\begin{align*}
    \sum_{i \in [N]} \mathbb{E}_{\mu_0, \mu_1} \left[ \left\| P_{\Thetav_i} \left( \cdot \mid \tau_{i,0}^{(j)}, \tau_{i,1}^{(j)} \right) - P_{\Thetav^\star_i} \left( \cdot \mid \tau_{i,0}^{(j)}, \tau_{i,1}^{(j)} \right) \right\|^2 \right] \leq \delta,
\end{align*}
then, for some constant \(C > 0\), it holds that \(\dist(\Bv, \Bv^\diamond) \leq \frac{\|\Thetav - \Thetav^\star\|_F^2}{(\delta')^2} \leq C \frac{\delta}{N \nu}\) where \(\nu = \sigma_k\left(\frac{(\Thetav^\star)^T \Thetav^\star}{N}\right)\).

\end{lemma}

\begin{proof}
From the Mean Value Theorem, there exists a constant \(C > 0\) such that 
\begin{align*}
    \|\Thetav - \Thetav^\star\|_F^2 \leq C \sum_{i \in [N]} \Eb_{\mu_0, \mu_1} \left[ \norm{P_{\Thetav_i} \left( \cdot \mid \tau_{i, 0}^{(j)}, \tau_{i,1}^{(j)} \right) - P_{\Thetav^\star_i} \left( \cdot \mid \tau_{i, 0}^{(j)}, \tau_{i,1}^{(j)} \right)}^2 \right] \leq C\delta.
\end{align*}
Define \(\delta' := \min_{1 \leq i \leq k, k+1 \leq j \leq \min\{d_1, Nd_2\}} \left|\sigma_i(\Thetav^\star) - \sigma_j(\Thetav)\right|\). Then, we observe that
\begin{align*}
    \delta' = \min_{1 \leq i \leq k, k+1 \leq j \leq \min\{d_1, Nd_2\}} \left|\sigma_i(\Thetav^\star) - \sigma_j(\Thetav)\right| = \sigma_k(\Thetav^\star).
\end{align*}
Next, by applying the Davis-Kahan Theorem, we obtain
\begin{align*}
    \text{dist}^2(\Bv, \Bv^\diamond) \leq \frac{\|\Thetav - \Thetav^\star\|_F^2}{(\delta')^2} \leq C \frac{\delta}{\sigma_k^2(\Thetav^\star)} = C \frac{\delta}{N\nu}.
\end{align*}
This is the desired result.
\end{proof}

\section{Training from Scratch}\label{sec:apx.A.1}

In this section, we present our theoretical analysis, focusing on the case where the initialization is set to zero, i.e., the initial parameter matrix is \(\Thetav^{\text{0}} = \mathbf{0}^{d_1 \cdot d_2}\). Consequently, \(\Delta \Thetav^\star = \Thetav^\star\).

With zero initialization, we define the class of reward functions \(\mathcal{G}'_{\mathbf{r}}\), in which a low-rank adaptation matrix with shared representations is learned as the parameter matrix for each individual reward function. Specifically, \(\mathcal{G}'_{\mathbf{r}}\) is defined as:
\[
\mathcal{G}'_{\rbm} = \left\{ \left( r_{\boldsymbol{\Theta}_i}(\cdot) \right)_{i \in [N]} \,\Big|\, \Thetav \in \mathbb{R}^{d_1 \cdot N d_2},\ \text{rank}(\Thetav) = k,\ \|\Thetav_i\|_F \leq B,\ \forall i \in [N] \right\},
\]
where \(\Thetav\) denotes the aggregated parameter matrix across all individuals, subject to a rank constraint \(k\). Additionally, each individual parameter matrix \(\Thetav_i\) satisfies the Frobenius norm constraint \(\|\Thetav_i\|_F \leq B\). We state our theoretical results under zero initialization as follows:

\begin{restatable}{theorem}{labelcorrect}
    \label{cor:label-is-correct}
    \emph{(Closeness between { \(\hat{\Bv}\)} and { \(\Bv^\diamond\)}).}
    For any \(\delta \in (0,1]\), with probability at least \(1 - \delta\), it holds that
    {
    \begin{align*}
        \dist(\hat{\Bv}, \Bv^\diamond) \leq c_1\sqrt{
        \frac{1}{NN_p\nu} \log(\Nc_{\Gc'_{\rbm}}(1/(NN_p))/ \delta)+\frac{1}{\nu}\sqrt{\frac{\tail}{N}}}.
    \end{align*}}
where \(c_1 > 0\) is a constant, \(\nu\) is the condition number as defined earlier, \(\tail\) represents the aggregate tail singular values and \(\mathcal{N}_{\mathcal{G}'_{\rbm}}(\cdot)\) denotes the beacketing number of the function class \(\mathcal{G}'_{\rbm}\).
\end{restatable}

\begin{restatable}{theorem}{diverse}
\label{thm:diverse-newmodel}
\emph{(Individual Expected Value Function Gap).} For any user \(i \in [N]\) and any $\delta \in (0, 1]$,
set \(\zeta\) in \Cref{eqn:confidenceset-alg1-1} as
\begin{equation}
    \begin{aligned}
        \zeta = \sqrt{\frac{c_3}{2L_1^2} \left(
        \min_{i\in[N]}\left\| \Thetav_i^\star - \Thetav_i^\diamond \right\|_F^2 +
        \frac{\log\left( \mathcal{N}_{\mathcal{G}'_{\rbm}}(1/(N N_p))/ \delta \right)}{N N_p \nu}+
        \frac{1}{\nu} \sqrt{ \frac{ \tail }{ N } } +
        \frac{ k d_2 + \log(N/\delta) }{ N_p }
    \right)},\label{def:zeta}
    \end{aligned}
\end{equation}
where $c_3 > 0$ is a constant. Then,  with probability at least $1-\delta$, the output policy $\hat{\pi}_i$ for client \(i\) satisfies 
\begin{align*}
    &J(\pi_{i, \text{tar}}; r^\star_i) - J(\hat{\pi}_i; r^\star_i)\\ 
    &\leq 
    \sqrt{c_3 \left(
        \left\| \Thetav_i^\star - \Thetav_i^\diamond \right\|_F^2 +
        \frac{\log\left( \mathcal{N}_{\mathcal{G}'_{\rbm}}(1/(N N_p))/ \delta \right)}{N N_p \nu}+
        \frac{1}{\nu} \sqrt{ \frac{ \tail }{ N } } +
        \frac{ k d_2 + \log(N/\delta) }{ N_p }
    \right)}.
\end{align*}  

\end{restatable}

\begin{restatable}{corollary}{average}
\label{thm:averaged-gap}
\emph{(Averaged Expected Value Function Gap).} For any $\delta \in (0, 1]$,  set \(\zeta\) as in \Cref{def:zeta}. If \(N \geq \mu^2 \tail\), then, with probability at least $1-\delta$, the output policies $\{\hat{\pi}_i\}_{i=1}^N$ satisfy the following inequality:
\begin{align*}
    \frac{1}{N}\sum_{i=1}^{N} \left(J(\pi_{i, \text{tar}}; r^\star_i) - J(\hat{\pi}_i; r^\star_i)\right) \leq c_4\sqrt{\frac{\log(\Nc_{\Gc'_{\rbm}}(1/(NN_p))/ \delta)}{NN_p}} + \sqrt{\frac{\tail}{N}},
\end{align*}
where $c_4 > 0$ is a constant.
\end{restatable}

% \subsection{Proof of \Cref{cor:label-is-correct}}

\subsection[Proof of Theorem~\ref{cor:label-is-correct}]{Proof of \Cref{cor:label-is-correct}}

\labelcorrect*
\begin{proof}
Consider the events \(\mathcal{E}_1\) and \(\mathcal{E}_2\) defined by the satisfaction of the conditions in Lemma~\ref{lemma:mle} and Lemma~\ref{lemma:l2distance}, respectively, with the confidence parameter adjusted to \(\delta \leftarrow \delta / 2\). This adjustment guarantees that \(\mathbb{P}(\mathcal{E}_1 \cap \mathcal{E}_2) \geq 1 - \delta\). Consequently, we conduct our analysis conditioned on the event \(\mathcal{E}_1 \cap \mathcal{E}_2\).

From \Cref{lemma:4} we have
\begin{align*}
    \sum_{i \in [N]} \sum_{j \in [N_p]} \log P_{\Thetav^\star_i}\left(o_i^{(j)} \mid \tau_{i,0}^{(j)}, \tau_{i,1}^{(j)}\right) 
    &\leq \sum_{i \in [N]} \sum_{j \in [N_p]} \log P_{\Thetav^\diamond_i}\left(o_i^{(j)} \mid \tau_{i,0}^{(j)}, \tau_{i,1}^{(j)}\right) + c N_p \sqrt{N \, \tail}.
\end{align*}

Using the definition of \(\hat{\Thetav}\) gives:
\begin{align*}
    \sum_{i \in [N]} \sum_{j \in [N_p]} \log P_{\Thetav^\star_i}\left(o_i^{(j)} \mid \tau_{i,0}^{(j)}, \tau_{i,1}^{(j)}\right) 
    &\leq \sum_{i \in [N]} \sum_{j \in [N_p]} \log P_{\hat{\Thetav}_i}\left(o_i^{(j)} \mid \tau_{i,0}^{(j)}, \tau_{i,1}^{(j)}\right) + c N_p \sqrt{N \, \tail}.
\end{align*}

Therefore, it follows that:
\begin{align*}
    \sum_{i \in [N]} \sum_{j \in [N_p]} \log \left( \frac{P_{r^\star_i}\left(o^{(j)} \mid \tau_{0}^{(j)}, \tau_{1}^{(j)}\right)}{P_{r_i}\left(o^{(j)} \mid \tau_{0}^{(j)}, \tau_{1}^{(j)}\right)} \right) 
    &\leq \sum_{i \in [N]} \sum_{j \in [N_p]} \log \left( \frac{P_{r^\diamond_i}\left(o^{(j)} \mid \tau_{0}^{(j)}, \tau_{1}^{(j)}\right)}{P_{r_i}\left(o^{(j)} \mid \tau_{0}^{(j)}, \tau_{1}^{(j)}\right)} \right) + c N_p \sqrt{N \, \tail} \\
    &\leq \log \left( \frac{\mathcal{N}_{\mathcal{G}'_{\mathbf{r}}}\left(\frac{1}{N N_p}\right)}{\delta} \right) + c N_p \sqrt{N \, \tail}.
\end{align*}
By \Cref{lemma:l2distance} we have:
\begin{align*}
    &\frac{1}{N} \sum_{i \in [N]} \mathbb{E}_{\mu_0, \mu_1} \left[ \left\| P_{\Thetav_i}\left( \cdot \mid \tau_{i,0}^{(j)}, \tau_{i,1}^{(j)} \right) - P_{\Thetav^\star_i}\left( \cdot \mid \tau_{i,0}^{(j)}, \tau_{i,1}^{(j)} \right) \right\|_1^2 \right] \\
    &\leq \frac{C_2}{N N_p} \left( \sum_{i \in [N]} \sum_{j \in [N_p]} \log \left( \frac{P_{\Thetav^\star_i}\left(o_i^{(j)} \mid \tau_{i,0}^{(j)}, \tau_{i,1}^{(j)}\right)}{P_{\Thetav_i}\left(o_i^{(j)} \mid \tau_{i,0}^{(j)}, \tau_{i,1}^{(j)}\right)} \right) + \log \left( \frac{\mathcal{N}_{\mathcal{G}'_{\mathbf{r}}}\left(\frac{1}{N N_p}\right)}{\delta} \right) \right) \\
    &\leq \frac{C_2}{N N_p} \left( C_1 \log \left( \frac{\mathcal{N}_{\mathcal{G}'_{\mathbf{r}}}\left(\frac{1}{N N_p}\right)}{\delta} \right) + c N_p \sqrt{N \, \tail} + \log \left( \frac{\mathcal{N}_{\mathcal{G}'_{\mathbf{r}}}\left(\frac{1}{N N_p}\right)}{\delta} \right) \right) \\
    &= \frac{C_3}{N N_p} \log \left( \frac{\mathcal{N}_{\mathcal{G}'_{\mathbf{r}}}\left(\frac{1}{N N_p}\right)}{\delta} \right) + C_4 \sqrt{\frac{\tail}{N}},
\end{align*}
for any \(\mathbf{r}_{\Thetav} \in \mathcal{R}(\hat{\mathcal{D}})\), where \(C_3 = C_2(C_1 + 1)\).
By the mean value theorem, for any \(\mathbf{r}_{\Thetav} \in \mathcal{R}(\hat{\mathcal{D}})\), we obtain:
\begin{equation}
\begin{aligned}
    &\frac{1}{N} \sum_{i \in [N]} \mathbb{E}_{\mu_0, \mu_1} \left[ \left| \left( r_{\Thetav_i}(\tau_{i, 0}) - r_{\Thetav_i}(\tau_{i, 1}) \right) - \left( r_{i}^\star(\tau_{i, 0}) - r_{i}^\star(\tau_{i, 1}) \right) \right|^2 \right] \\
    &\leq \frac{\kappa^2}{N} \sum_{i \in [N]} \mathbb{E}_{\mu_0, \mu_1} \left[ \left\| P_{\Thetav_i} ( \cdot \mid \tau_{i, 0}^{(j)}, \tau_{i,1}^{(j)}, i) - P_{\Thetav_i^\star} ( \cdot \mid \tau_{i, 0}^{(j)}, \tau_{i,1}^{(j)}, i) \right\|_1^2 \right] \\
    &\leq \frac{C_3 \kappa^2}{N N_p} \log \left( \frac{\mathcal{N}_{\mathcal{G}'_{\mathbf{r}}}(1 / (N N_p))}{\delta} \right) + C_4 \kappa^2 \sqrt{\frac{\tail}{N}}.
\end{aligned}
\label{eqn:reward-bound-thm1}    
\end{equation}
Therefore, combining with \Cref{lemma:6} gives
\begin{align*}
\text{dist}^2(\Bv, \Bv^\diamond) \leq 
\frac{CC_3}{NN_p\nu} \log(\Nc_{\Gc'_{\rbm}}(1/(NN_p))/ \delta)+\frac{CC_4}{\nu}\kappa^2\sqrt{\frac{\tail}{N}}.
\end{align*}
Also, we obtain
\begin{align*}
\|\Bv-\Bv^\diamond\|^2_F \leq \text{dist}^2(\Bv, \Bv^\diamond) \leq 
2\frac{CC_3}{NN_p\nu} \log(\Nc_{\Gc'_{\rbm}}(1/(NN_p))/ \delta)+2\frac{CC_4}{\nu}\kappa^2\sqrt{\frac{\tail}{N}}.
\end{align*}
This proves the theorem.
\end{proof}

\subsection[Proof of Theorem~\ref{thm:diverse-newmodel}]{Proof of \Cref{thm:diverse-newmodel}}

Before formally proving \Cref{thm:diverse-newmodel}, we present the following theorem as an intermediate result.
\begin{theorem}\label{thm:step2}
For any \(\delta \in (0, 1]\), with probability at least \(1- \delta\), it holds that
\begin{align*}
    &\frac{1}{N_p}\sum_{j\in[N_p]}\left|(r_{\hat{\Thetav}_i}(\tau_{i, 0}^{(j)}) - r_{\hat{\Thetav}_i}(\tau_{i, 1}^{(j)})) - (r_{\Thetav_i^\star}(\tau_{i, 0}^{(j)}) - r_{\Thetav_i^\star}(\tau_{i, 1}^{(j)}))\right|^2 \\
    \quad&\leq C_8\left(
    \left\| \Thetav_i^\star - \Thetav_i^\diamond \right\|_F^2 +
    \frac{1}{N N_p \nu} \log\left( \frac{ \mathcal{N}_{\mathcal{G}'_{\rbm}}(1/(N N_p)) }{ \delta } \right) +
    \frac{1}{\nu} \sqrt{ \frac{ \tail }{ N } } +
    \frac{ k d_2 + \log(N/\delta) }{ N_p }
    \right),
\end{align*}
where \(C_8>0\) is a constant.
\end{theorem}\label{thm:C.1}

\begin{proof}
Recall that for a function \( r_{\Thetav} \) parameterized by the matrix \(\Thetav \in \mathbb{R}^{d_1 \cdot Nd_2}\), we use \( r_{\theta} \) to denote the same function parameterized by the vector \(\theta\), where \(\theta = \text{vec}(\Thetav)\). To begin, by leveraging the continuity of \(r_{\Thetav}(\cdot)\), we can establish the following inequality:
\begin{equation}
\begin{aligned}
    &\left|(r_{\hat{\Thetav}_i}(\tau_{i, 0}^{(j)}) - r_{\hat{\Thetav}_i}(\tau_{i, 1}^{(j)})) - (r_{\Thetav_i^\star}(\tau_{i, 0}^{(j)}) - r_{\Thetav_i^\star}(\tau_{i, 1}^{(j)}))\right|^2_F \\
    &\leq
    2\left|(r_{\hat{\Thetav}_i}(\tau_{i, 0}^{(j)}) - r_{\hat{\Thetav}_i}(\tau_{i, 1}^{(j)})) - (r_{\Thetav_i^\diamond}(\tau_{i, 0}^{(j)}) - r_{\Thetav_i^\diamond}(\tau_{i, 1}^{(j)}))\right|^2 + 4L\|\Thetav_i^\star - \Thetav_i^\diamond\|^2\\
    & = 2\left|(r_{\hat{\theta}_i}(\tau_{i, 0}^{(j)}) - r_{\hat{\theta}_i}(\tau_{i, 1}^{(j)})) - (r_{\theta_i^\diamond}(\tau_{i, 0}^{(j)}) - r_{\theta_i^\diamond}(\tau_{i, 1}^{(j)}))\right|^2 + 4L\|\Thetav_i^\star - \Thetav_i^\diamond\|^2_F.\label{equ:C.1}
\end{aligned}
\end{equation}

Next, we focus on obtaining an upper bound for the first part of the right-hand side of the above inequality. Using the Lagrange form of the remainder in the Taylor expansion of \(r_{\hat{\theta}_i}(\tau_{i, 0}^{(j)})\), we get
\[
    r_{\theta^\diamond_i}(\tau_{i, 0}^{(j)}) - r_{\hat{\theta}_i}(\tau_{i, 0}^{(j)}) = 
    \nabla_{\theta}r_{\bar{\theta}_i}(\tau_{i,0}^{(j)})^\top (\theta^\diamond_i-\hat{\theta}_i).
\]
Therefore, there exist \(\bar{\theta}_0\) and \(\bar{\theta}_1\) such that
\begin{align*}
    (r_{\theta^\diamond_i}(\tau_{i, 0}^{(j)}) - r_{\theta^\diamond_i}(\tau_{i, 1}^{(j)})) - (r_{\hat{\theta}_i}(\tau_{i, 0}^{(j)}) - r_{\hat{\theta}_i}(\tau_{i, 1}^{(j)}))
    =
    \left(\nabla_{\theta}r_{\bar{\theta}_0}(\tau_{i,0}^{(j)})-\nabla_{\theta}r_{\bar{\theta}_1}(\tau_{i,1}^{(j)})\right)^\top (\theta^\diamond_i-\hat{\theta}_i).
\end{align*}
Then, we obtain the following:
\begin{equation}
\begin{aligned}
    &\left|(r_{\theta^\diamond_i}(\tau_{i, 0}^{(j)}) - r_{\theta^\diamond_i}(\tau_{i, 1}^{(j)})) - (r_{\hat{\theta}_i}(\tau_{i, 0}^{(j)}) - r_{\hat{\theta}_i}(\tau_{i, 1}^{(j)}))\right|^2\\
    &\leq
    2\left|\left(\nabla_{\theta}r_{\hat{\theta}_i}(\tau_{i,0}^{(j)})-\nabla_{\theta}r_{\hat{\theta}_i}(\tau_{i,1}^{(j)})\right)^\top (\theta^\diamond_i-\hat{\theta}_i)\right|^2+
    4\left|\left(\nabla_{\theta}r_{\bar{\theta}_0}(\tau_{i,0}^{(j)})-\nabla_{\theta}r_{\hat{\theta}_i}(\tau_{i,0}^{(j)})\right)^\top (\theta^\diamond_i-\hat{\theta}_i)\right|^2\\
    &\quad+4\left|\left(\nabla_{\theta}r_{\bar{\theta}_1}(\tau_{i,1}^{(j)})-\nabla_{\theta}r_{\hat{\theta}_i}(\tau_{i,1}^{(j)})\right)^\top (\theta^\diamond_i-\hat{\theta}_i)\right|^2
    \\
    &\leq
    2\underbrace{\left|\left(\nabla_{\theta}r_{\hat{\theta}_i}(\tau_{i,0}^{(j)})-\nabla_{\theta}r_{\hat{\theta}_i}(\tau_{i,1}^{(j)})\right)^\top (\theta^\diamond_i-\hat{\theta}_i)\right|^2}_{{\Ac_{i,j}}}
    +16L_1\underbrace{\|\theta_i^\diamond-\hat{\theta}_i\|^2}_{{\Bc_i}}.\label{ineq:C.4.1}
\end{aligned}
\end{equation}

Our remaining proof contains three major steps: (1) \textbf{Step 1:} bounding the summation of \(\Ac_{i,j}\) over \(j\); (2) \textbf{Step 2:} bounding the term \(\Bc_{i}\); and (3) \textbf{Step 3:} combining the bounds for \(\Ac_{i,j}\) and \(\Bc_{i}\) to obtain the final result. We now proceed with the proof of the first step.

\textbf{Step 1: Bounding the summation of \(\Ac_{i,j}\) over \(j\).}

Regarding the term \(\mathcal{A}_{i,j}\), let us denote \(w^\diamond_i = \mtov(\Wv_i^\diamond)\). Then, we have
\begin{align*}
    &\left(\nabla_{\theta} r_{\hat{\theta}_i}(\tau_{i,0}^{(j)}) - \nabla_{\theta} r_{\hat{\theta}_i}(\tau_{i,1}^{(j)})\right)^\top (\theta^\diamond_i - \hat{\theta}_i) \\
    &= \left(\nabla_{\theta} r_{\hat{\theta}_i}(\tau_{i,0}^{(j)}) - \nabla_{\theta} r_{\hat{\theta}_i}(\tau_{i,1}^{(j)})\right)^\top \big(\theta^\diamond_i + (\ikp\hat{\Bv}) w^\diamond_i - (\ikp\hat{\Bv}) w^\diamond_i - \hat{\theta}_i\big).
\end{align*}
Utilizing the fact that \(\theta^\diamond_i = (\ikp\Bv^\diamond) w_i^\diamond\), it follows that
\begin{align*}
    \mathcal{A}_{i,j} &\leq 2 \left| \left(\nabla_{\theta} r_{\hat{\theta}_i}(\tau_{i,0}^{(j)}) - \nabla_{\theta} r_{\hat{\theta}_i}(\tau_{i,1}^{(j)})\right)^\top \left( (\ikp\Bv^\diamond) w_i^\diamond - (\ikp\hat{\Bv}) w_i^\diamond \right) \right|^2 \\
    &\quad + 2 \left| \left(\nabla_{\theta} r_{\hat{\theta}_i}(\tau_{i,0}^{(j)}) - \nabla_{\theta} r_{\hat{\theta}_i}(\tau_{i,1}^{(j)})\right)^\top \left( (\ikp\hat{\Bv}) w_i^\diamond - (\ikp\hat{\Bv}) \hat{w}_i \right) \right|^2 \\
    &\underset{(i)}{\leq} 4L \|(\ikp\Bv^\diamond) - (\ikp\hat{\Bv})\|^2 \|w_i^\diamond\|^2 + 2 \left| \left(\nabla_{\theta} r_{\hat{\theta}_i}(\tau_{i,0}^{(j)}) - \nabla_{\theta} r_{\hat{\theta}_i}(\tau_{i,1}^{(j)})\right)^\top (\ikp\hat{\Bv}) (w_i^\diamond - \hat{w}_i) \right|^2 \\
    &\underset{(ii)}{=} 4L \|\Bv^\diamond - \hat{\Bv}\|^2 \|\Wv_i^\diamond\|_F^2 + 2 \left| \left(\nabla_{\theta} r_{\hat{\theta}_i}(\tau_{i,0}^{(j)}) - \nabla_{\theta} r_{\hat{\theta}_i}(\tau_{i,1}^{(j)})\right)^\top (\ikp\hat{\Bv}) (w_i^\diamond - \hat{w}_i) \right|^2,
\end{align*}
where inequality \((i)\) follows from the \(L\)-Lipschitz continuity of the function \(r_{\theta}(\cdot)\) with respect to \(\theta\), and equality \((ii)\) is derived from the facts that \(\|(\ikp\Bv^\diamond) - (\ikp\hat{\Bv})\|^2 = \|\Bv^\diamond - (\ikp\hat{\Bv})\|^2\) and \(\|w_i^\diamond\|^2 = \|\Wv_i^\diamond\|_F^2\). Next, we define
\begin{align*}
    \hat{\Sigmav}_i = \frac{1}{N_p}\sum_{j\in{N_p}}(\ikp\hat{\Bv})^\top \left(\nabla_{\theta}r_{\hat{\theta}_i}(\tau_{i,0}^{(j)})-\nabla_{\theta}r_{\hat{\theta}_i}(\tau_{i,1}^{(j)})\right)\left(\nabla_{\theta}r_{\hat{\theta}_i}(\tau_{i,0}^{(j)})-\nabla_{\theta}r_{\hat{\theta}_i}(\tau_{i,1}^{(j)})\right)^\top(\ikp\hat{\Bv}).
\end{align*}
Following the definition of $\hat{\Sigmav}_i$, we further derive the following inequality:
\begin{align}
    \frac{1}{N_p}\sum_{j\in{N_p}}\Ac_{i,j}\leq
4L \|\Bv^\diamond-\hat{\Bv}\|^2\|\Wv_i^\diamond\|_F^2+2\|w_i^\diamond-\hat{w}_i\|^2_{\hat{\Sigmav}_i}.\label{equ:A.4}
\end{align}
Now, we consider the following optimization problem:
\begin{align*}
    \underset{w_i}{\max} \quad f(w_i) :=  \frac{1}{N_p} \sum_{j \in [N_p]} \log P_{(\ikp\hat{\Bv})w_i}(o_{i}^{(j)} \mid  \tau_{i, 0}^{(j)}, \tau_{i, 1}^{(j)}). 
\end{align*}
The solution to this optimization problem is given by $\hat{w}_i = \underset{w}{\argmax} \, f(w_i)$. To proceed with the analysis, let us denote \(x_i^{(j)} = r_{(\ikp\hat{\Bv})w_i}(\tau_{i,0}^{(j)}) - r_{(\ikp\hat{\Bv})w_i}(\tau_{i,1}^{(j)})\). Using this notation, the gradient of the objective function can be expressed as follows:
\begin{align*}
    \nabla f(w_i) =& \frac{1}{N_p} \sum_{j \in [N_p]} \biggl(\frac{\Phi'(x_i^{(j)}) }{\Phi(x_i^{(j)})} \pmb{1}(o_i^{(j)} = 0)  - \frac{\Phi'(-x_i^{(j)}) }{\Phi(-x_i^{(j)})} \pmb{1}(o_i^{(j)} = 1)  \biggr)\\
    \qquad&\cdot(\ikp\hat{\Bv})^\top\big(\nabla r_{(\ikp\hat{\Bv})w_i}(\tau_{i,0}^{(j)})-\nabla r_{(\ikp\hat{\Bv})w_i}(\tau_{i,1}^{(j)})\big),
\end{align*}
and
\begin{equation}\label{ineq:A.3}
    \begin{split}
        \nabla^2f(w_i) =& \frac{1}{N_p} \sum_{j \in [N_p]} \biggl(\frac{\Phi'(x_i^{(j)})}{\Phi(x_i^{(j)})} \pmb{1}(o_i^{(j)} = 0) - \frac{\Phi'(-x_i^{(j)}) }{\Phi(-x_i^{(j)})} \pmb{1}(o_i^{(j)} = 1)  \biggr)\\
        \qquad&\cdot(\ikp\hat{\Bv})^\top\big(\nabla^2 r_{(\ikp\hat{\Bv})w_i}(\tau_{i,0}^{(j)})-\nabla^2 r_{(\ikp\hat{\Bv})w_i}(\tau_{i,1}^{(j)})\big)(\ikp\hat{\Bv})\\
    \qquad&+
    \frac{1}{N_p} \sum_{j \in [N_p]} \biggl(\frac{\Phi''(x_i^{(j)})\Phi(x_i^{(j)})-\Phi'(x_i^{(j)})^2}{\Phi(x_i^{(j)})^2} \pmb{1}(o_i^{(j)} = 0)\\
    \qquad&+ \frac{\Phi''(-x_i^{(j)})\Phi(-x_i^{(j)})-\Phi'(-x_i^{(j)})^2}{\Phi(-x_i^{(j)})^2}\pmb{1}(o_i^{(j)} = 1)  \biggr)\\ 
    \qquad&\cdot 
    (\ikp\hat{\Bv})^\top\big(\nabla r_{(\ikp\hat{\Bv})w_i}(\tau_{i,0}^{(j)})-\nabla r_{(\ikp\hat{\Bv})w_i}(\tau_{i,1}^{(j)})\big)\\
    \qquad&\cdot\big(\nabla r_{(\ikp\hat{\Bv})w_i}(\tau_{i,0}^{(j)})-\nabla r_{(\ikp\hat{\Bv})w_i}(\tau_{i,1}^{(j)})\big)^\top(\ikp\hat{\Bv}).
    \end{split}
\end{equation}
From the Lagrange form of the remainder in the Taylor expansion, there exist \(\bar{w}_i\) such that
\begin{align}
    f(\hat{w}_i) = f(w_i^\diamond) + \nabla f(w_i^\diamond)^\top(\hat{w}_i-w_i^\diamond) + (\hat{w}_i-w_i^\diamond)^\top\nabla^2f(\bar{w}_i)(\hat{w}_i-w_i^\diamond).\label{equ:A.4.1}
\end{align}
To handle the \((\hat{w}_i-w_i^\diamond)^\top\nabla^2f(\bar{w}_i)(\hat{w}_i-w_i^\diamond)\) term, we define
\begin{align*}
    \Sigmav_i^\diamond =& \frac{1}{N_p} \sum_{j \in N_p} (\ikp \hat{\Bv})^\top \left( \nabla r_{(\ikp \hat{\Bv}) w_i^\diamond}(\tau_{i,0}^{(j)}) - \nabla r_{(\ikp \hat{\Bv}) w_i^\diamond}(\tau_{i,1}^{(j)}) \right)\cdot \\
    &\quad  \left( \nabla r_{(\ikp \hat{\Bv}) w_i^\diamond}(\tau_{i,0}^{(j)}) - \nabla r_{(\ikp \hat{\Bv}) w_i^\diamond}(\tau_{i,1}^{(j)}) \right)^\top (\ikp \hat{\Bv}),
\end{align*}
and let \(c_1\) and \(c'_1\) be the maximum and minimum positive constants, respectively, such that for any \(i\), \(\|w_i\| \leq B\), and any vector \(u\), the following inequality holds:
\begin{equation}
\begin{aligned}
    c_1 \, u^\top \Sigmav_i^\diamond u \leq & \frac{1}{N_p} \sum_{j \in N_p} u^\top (\ikp \hat{\Bv})^\top \left( \nabla r_{(\ikp \hat{\Bv}) w_i}(\tau_{i,0}^{(j)}) - \nabla r_{(\ikp \hat{\Bv}) w_i}(\tau_{i,1}^{(j)}) \right) \\
    &\quad \cdot \left( \nabla r_{(\ikp \hat{\Bv}) w_i}(\tau_{i,0}^{(j)}) - \nabla r_{(\ikp \hat{\Bv}) w_i}(\tau_{i,1}^{(j)}) \right)^\top (\ikp \hat{\Bv}) u \leq c'_1 \, u^\top \Sigmav_i^\diamond u. \label{def:c1}
\end{aligned}
\end{equation}
Combining this with inequality~\eqref{ineq:A.3}, we obtain:
\begin{align*}
(\hat{w}_i-w_i^\diamond)^\top\nabla^2f(\bar{w}_i)(\hat{w}_i-w_i^\diamond)
&\leq
\frac{1}{N_p}(\hat{w}_i-w_i^\diamond)^\top \sum_{j \in [N_p]} \biggl(\frac{\Phi'(x_i^{(j)})}{\Phi(x_i^{(j)})} \pmb{1}(o_i^{(j)} = 0) - \frac{\Phi'(-x_i^{(j)}) }{\Phi(-x_i^{(j)})} \pmb{1}(o_i^{(j)} = 1)  \biggr)\\
&\cdot(\ikp\hat{\Bv})^\top\big(\nabla^2 r_{(\ikp\hat{\Bv})w_i}(\tau_{i,0}^{(j)})-\nabla^2 r_{(\ikp\hat{\Bv})w_i}(\tau_{i,1}^{(j)})\big)(\ikp\hat{\Bv})(\hat{w}_i-w_i^\diamond)\\
\qquad&
-c_1c_2
(\hat{w}_i-w_i^\diamond)^\top\Sigmav^\diamond_i(\hat{w}_i-w_i^\diamond)
,
\end{align*}
where \(c_2= \min_{x}\left(\frac{\Phi'(x)^2 - \Phi''(x)\Phi(x)}{\Phi(x)^2}\right)\).  Then, from the smoothness of \(r_{\theta}\) we have:
\begin{align*}
    &\frac{1}{N_p}\sum_{j\in[N_p]}(\hat{w}_i-w_i^\diamond)^\top(\ikp\hat{\Bv})^\top
    \big(\nabla^2 r_{(\ikp\hat{\Bv})w_i}(\tau_{i,0}^{(j)})-\nabla^2 r_{(\ikp\hat{\Bv})w_i}(\tau_{i,1}^{(j)})\big) (\ikp\hat{\Bv})(\hat{w}_i-w_i^\diamond) \\
    \qquad&\leq
    L_2(\hat{w}_i-w_i^\diamond)^\top(\ikp\hat{\Bv})^\top(\ikp\hat{\Bv})(\hat{w}_i-w_i^\diamond)\\
    \qquad&=
    L_2\|\hat{w}_i-w_i^\diamond\|^2
\end{align*}
Let \(c_3 = \max_{x}\left(\Phi'(x)/\Phi(x)\right)\) we have
\begin{align*}
    (\hat{w}_i-w_i^\diamond)^\top\nabla^2f(\bar{w}_i)(\hat{w}_i-w_i^\diamond)\leq 
-c_1c_2
(\hat{w}_i-w_i^\diamond)^\top\Sigmav^\diamond_i(\hat{w}_i-w_i^\diamond)+c_3L_2\|\hat{w}_i-w_i^\diamond\|^2.
\end{align*}
Combining with \Cref{equ:A.4.1} gives
\begin{align}
    c_1c_2
(\hat{w}_i-w_i^\diamond)^\top\Sigmav^\diamond_i(\hat{w}_i-w_i^\diamond)-c_3L_2\|\hat{w}_i-w_i^\diamond\|^2\leq 
\nabla f(w_i^\diamond)^\top(w_i^\diamond-\hat{w}_i).\label{ineq:A.5}
\end{align}
From the smoothness of \(f(w)\), we have
\begin{align*}
    f(\hat{w}_i) \leq f(w_i^\diamond) + \nabla f(w_i^\diamond)^\top(\hat{w}_i-w_i^\diamond) + \frac{L_2'}{2}\|\hat{w}_i-w_i^\diamond\|^2,
\end{align*}
using the fact \(f(\hat{w}_i)\geq f(w_i^\diamond)\) we have
\begin{align}
    \|\hat{w}_i-w_i^\diamond\|^2\leq \frac{2}{L_2'}\nabla f(w_i^\diamond)^\top(w_i^\diamond-\hat{w}_i).\label{ineq:C.8}
\end{align}
Then, combining the above inequality with \Cref{ineq:A.5}, we conclude that for any \(\lambda>0\) the following inequality holds
\begin{equation}
    \begin{aligned}
    c_1c_2\|\hat{w}_i-w_i^\diamond\|_{\Sigmav^\diamond}^2
    &\leq \left(1+2c_3\frac{L_2}{L_2'}\right)\nabla f(w_i^\diamond)^\top(w_i^\diamond-\hat{w}_i)\\
&\leq\left(1+2c_3\frac{L_2}{L_2'}\right)|\nabla f(w_i^\diamond)^\top(w_i^\diamond-\hat{w}_i)|\\
&\leq\left(1+2c_3\frac{L_2}{L_2'}\right)\|\nabla f(w_i^\diamond)\|_{(\Sigmav^\diamond+\lambda\Iv)^{-1}}\|w_i^\diamond-\hat{w}_i\|_{\Sigmav^\diamond+\lambda\Iv}.\label{ineq:A.7.1}
\end{aligned}
\end{equation} 
Observe that for any \(\lambda>0\), the introduced \(\lambda\Iv\) term will ensure \(\Sigmav^\diamond_i+\lambda\Iv\) is a full rank since \(\Sigmav^\diamond_i\) is a PSD matrix. For all \(i\), we define a random vector $V \in \Rb^{N_p}$ as follows:
\begin{align*}
V_{i,j}=\left\{\begin{array}{lll}
\frac{\Phi'(r_{\theta_i^\star}(\tau_{i,0}^{(j)})-r_{\theta_i^\star}(\tau_{i,1}^{(j)})) }{\Phi(r_{\theta_i^\star}(\tau_{i,0}^{(j)})-r_{\theta_i^\star}(\tau_{i,1}^{(j)}))} & \text { w.p. } & {\Phi(r_{\theta_i^\star}(\tau_{i,0}^{(j)})-r_{\theta_i^\star}(\tau_{i,1}^{(j)}))} \\
 - \frac{\Phi'(r_{\theta_i^\star}(\tau_{i,1}^{(j)})-r_{\theta_i^\star}(\tau_{i,0}^{(j)})) }{\Phi(r_{\theta_i^\star}(\tau_{i,1}^{(j)})-r_{\theta_i^\star}(\tau_{i,0}^{(j)}))} & \text { w.p. } & {\Phi(r_{\theta_i^\star}(\tau_{i,1}^{(j)})-r_{\theta_i^\star}(\tau_{i,0}^{(j)}))}
\end{array}\right.
\end{align*}
Also, define $V'_i \in \Rb^{N_p}$ as follows:
\begin{align*}
V'_{i,j}=\left\{\begin{array}{lll}
\frac{\Phi'(r_{(\ikp\hat{\Bv})w_i^\diamond}(\tau_{i,0}^{(j)})-r_{(\ikp\hat{\Bv})w_i^\diamond}(\tau_{i,1}^{(j)})) }{\Phi(r_{(\ikp\hat{\Bv})w_i^\diamond}(\tau_{i,0}^{(j)})-r_{(\ikp\hat{\Bv})w_i^\diamond}(\tau_{i,1}^{(j)}))} & \text { w.p. } & {\Phi(r_{\theta_i^\star}(\tau_{i,0}^{(j)})-r_{\theta_i^\star}(\tau_{i,1}^{(j)}))} \\
 - \frac{\Phi'(r_{(\ikp\hat{\Bv})w_i^\diamond}(\tau_{i,1}^{(j)})-r_{(\ikp\hat{\Bv})w_i^\diamond}(\tau_{i,0}^{(j)})) }{\Phi(r_{(\ikp\hat{\Bv})w_i^\diamond}(\tau_{i,1}^{(j)})-r_{(\ikp\hat{\Bv})w_i^\diamond}(\tau_{i,0}^{(j)}))} & \text { w.p. } & {\Phi(r_{\theta_i^\star}(\tau_{i,1}^{(j)})-r_{\theta_i^\star}(\tau_{i,0}^{(j)}))}
\end{array}\right.
\end{align*}
Therefore, \(\nabla f(w^\diamond_i)\) can be rewritten as
\begin{align*}
    \nabla f(w^\diamond_i) 
    =&
    \frac{1}{N_p} \sum_{j \in [N_p]} V'_{i,j}(\ikp\hat{\Bv})^\top\big(\nabla r_{(\ikp\hat{\Bv})w^\diamond_i}(\tau_{i,0}^{(j)})-\nabla r_{\Iv\otimes\Bv(\ikp\hat{\Bv})w^\diamond_i}(\tau_{i,1}^{(j)})\big)\\
    =&
    \frac{1}{N_p} \sum_{j \in [N_p]} (V'_{i,j}-V_{i,j})(\ikp\hat{\Bv})^\top\big(\nabla r_{(\ikp\hat{\Bv})w^\diamond_i}(\tau_{i,0}^{(j)})-\nabla r_{(\ikp\hat{\Bv})w^\diamond_i}(\tau_{i,1}^{(j)})\big)\\
    &+\frac{1}{N_p} \sum_{j \in [N_p]} V_{i,j}(\ikp\hat{\Bv})^\top\big(\nabla r_{(\ikp\hat{\Bv})w^\diamond_i}(\tau_{i,0}^{(j)})-\nabla r_{(\ikp\hat{\Bv})w^\diamond_i}(\tau_{i,1}^{(j)})\big).
\end{align*}
Then we obtain
\begin{equation}
\begin{aligned}
    &\|\nabla f(w_i^\diamond)\|_{(\Sigmav_i^\diamond + \lambda\Iv)^{-1}}\\
    &\leq
    \left\|\frac{1}{N_p} \sum_{j \in [N_p]} (V'_{i,j} - V_{i,j})(\ikp\hat{\Bv})^\top \big(\nabla r_{(\ikp\hat{\Bv})w^\diamond_i}(\tau_{i,0}^{(j)}) - \nabla r_{(\ikp\hat{\Bv})w^\diamond_i}(\tau_{i,1}^{(j)})\big)\right\|_{(\Sigmav_i^\diamond + \lambda\Iv)^{-1}} \\
    &\quad+ \left\|\frac{1}{N_p} \sum_{j \in [N_p]} V_{i,j}(\ikp\hat{\Bv})^\top \big(\nabla r_{(\ikp\hat{\Bv})w^\diamond_i}(\tau_{i,0}^{(j)}) - \nabla r_{(\ikp\hat{\Bv})w^\diamond_i}(\tau_{i,1}^{(j)})\big)\right\|_{(\Sigmav_i^\diamond + \lambda\Iv)^{-1}}.
    \label{ineq:A.7}
\end{aligned}
\end{equation}

Next, we bound the first term on the right-hand side of \Cref{ineq:A.7}. By the Mean Value Theorem, we have 
\(
\left| \frac{\Phi'(x)}{\Phi(x)} - \frac{\Phi'(y)}{\Phi(y)} \right| \leq \xi |x - y|,
\)
for \(x, y \in [-2R_{\max}, 2R_{\max}]\). Therefore, we can write:
\begin{align*}
    |V'_{i,j} - V_{i,j}| 
    &\leq \xi \left|r_{\theta_i^\star}(\tau_{i,0}^{(j)}) - r_{\theta_i^\star}(\tau_{i,1}^{(j)}) - r_{(\ikp\hat{\Bv})w^\diamond_i}(\tau_{i,0}^{(j)}) + r_{(\ikp\hat{\Bv})w^\diamond_i}(\tau_{i,1}^{(j)}) \right| \\
    &\overset{(i)}{\leq} 2L\xi \left\|\theta_i^\star - \theta_i^\diamond + \theta_i^\diamond - (\ikp\hat{\Bv})w^\diamond_i \right\| \\
    &\leq 2L\xi \left\|\theta_i^\star - \theta_i^\diamond\right\| + 2L\xi \left\|(\ikp\Bv^\diamond) - (\ikp\hat{\Bv})\right\| \cdot \left\|w_i^\diamond\right\|\\
    &=2L\xi\|\Thetav^\star_i-\Bv^\diamond\Wv_i^\diamond\|_F +2L\xi\left\|\Bv^\diamond-\hat{\Bv}\right\|\|\Wv_i^\diamond\|_F,
\end{align*}
where inequality \((i)\) follows from the \(L\)-Lipschitz continuity of \(r_{\Thetav}(\cdot)\). 

Then, we have
\begin{equation}
    \begin{aligned}
        &\bigg\|\frac{1}{N_p} \sum_{j \in [N_p]} (V'_{i,j} - V_{i,j})\hat{\Bv}^\top \big(\nabla r_{(\ikp\hat{\Bv})\Wv^\diamond_i}(\tau_{i,0}^{(j)}) - \nabla r_{(\ikp\hat{\Bv})\Wv^\diamond_i}(\tau_{i,1}^{(j)})\big)\bigg\|_{(\Sigmav_i^\diamond + \lambda\Iv)^{-1}}\\
    &\leq
    2CL\xi\|\Thetav^\star_i-\Bv^\diamond\Wv_i^\diamond\|_F +2CL\xi\big\|\Bv^\diamond-\hat{\Bv}\big\|\|\Wv_i^\diamond\|_F\label{ineq:A.8}
    \end{aligned}
\end{equation}
for constant \(C\).  

Next, we bound the second term on the right-hand side of \Cref{ineq:A.7}. Let \(V_i \in \Rb^{N_p}\) be the vector such that \([V_i]_j = V_{i,j}\) for all \(j \in [N_p]\) and we define
\begin{align*}
    \Mv_i :=\frac{1}{N_p^2}\Gv_i^\top(\ikp\hat{\Bv}) (\Sigmav_i^\diamond + \lambda\Iv)^{-1}(\ikp\hat{\Bv})^\top \Gv_i
\end{align*}
where
\begin{gather*}
\Gv_i = \begin{bmatrix}
        \nabla r_{(\ikp\hat{\Bv})w^\diamond_i}(\tau_{i,0}^{(1)}) - \nabla r_{(\ikp\hat{\Bv})w^\diamond_i}(\tau_{i,1}^{(1)}) &\cdots & \nabla r_{(\ikp\hat{\Bv})w^\diamond_i}(\tau_{i,0}^{(N_p)}) - \nabla r_{(\ikp\hat{\Bv})w^\diamond_i}(\tau_{i,1}^{(N_p)})
    \end{bmatrix}.
\end{gather*}
As shown in \citet{zhu2023principled}, the matrix \(\Mv_i\) satisfies the following properties:
\begin{align*}
    \operatorname{Tr}(\Mv_i) \leq \frac{d_2k}{N_p}, \qquad
    \operatorname{Tr}\left(\Mv_i^2\right) \leq \frac{d_2k}{N_p^2},\qquad
    \|\Mv_i\|_{F} \leq \frac{1}{N_p}.
\end{align*}
Furthermore, consider that the variables \(V_{i,j}\) are centered sub-Gaussian random variables, as \(\mathbb{E}[V_{i,j}] = 0\) and \(V_{i,j}\) are bounded. Consequently, by applying Bernstein's inequality, we obtain
\begin{equation}
    \begin{aligned}
        &\left\| \frac{1}{N_p} \sum_{j \in [N_p]} V_{i,j} (\ikp \hat{\Bv})^\top \left( \nabla r_{(\ikp \hat{\Bv}) w^\diamond_i}(\tau_{i,0}^{(j)}) - \nabla r_{(\ikp \hat{\Bv}) w^\diamond_i}(\tau_{i,1}^{(j)}) \right) \right\|_{(\Sigmav_i^\diamond + \lambda \Iv)^{-1}} \\
        &= \sqrt{V_i^\top \Mv_i V_i}\leq C_4 \sqrt{ \frac{k d_2 + \log(N/\delta)}{N_p} }, \label{ineq:A.9}
    \end{aligned}
\end{equation}
with probability at least \(1 - \delta/(2N)\), where \(C_4 > 0\) is constant. 

Subsequently, by substituting \Cref{ineq:A.8} and \Cref{ineq:A.9} into \Cref{ineq:A.7}, we obtain
\begin{align*}
    &\|\nabla f(\Wv_i^\diamond)\|_{(\Sigmav_i^\diamond + \lambda \Iv)^{-1}}\\
    &\leq 2CL\xi \left\| \Thetav_i^\star - \Bv^\diamond \Wv_i^\diamond \right\|_F 
    + 2CL\xi \left\| \Bv^\diamond - \hat{\Bv} \right\| \cdot \left\| \Wv_i^\diamond \right\|_F 
    + C_4 \sqrt{ \frac{kd_2 + \log(N/\delta)}{N_p} } \\
    &\leq 2CL\xi \left\| \Thetav_i^\star - \Thetav_i^\diamond \right\|_F 
    + 2CL\xi B \left\| \Bv^\diamond - \hat{\Bv} \right\| 
    + C_4 \sqrt{ \frac{kd_2 + \log(N/\delta)}{N_p} },
\end{align*}
where the final inequality leverages the facts that \(\Bv^\diamond \Wv_i^\diamond = \Thetav_i^\diamond\) and \(\|\Wv_i^\diamond\|_F \leq B\). 

Furthermore, utilizing \Cref{cor:label-is-correct}, we obtain
\begin{align}
    &\|\nabla f(\Wv_i^\diamond)\|_{(\Sigmav_i^\diamond + \lambda\Iv)^{-1}}^2\\
    &\leq
    C' \left(
        \left\| \Thetav_i^\star - \Thetav_i^\diamond \right\|_F^2 +
        \frac{1}{NN_p\nu} \log\left( \frac{ \Nc_{\Gc'_{\rbm}}(1/(NN_p)) }{ \delta } \right) +
        \frac{1}{\nu} \sqrt{ \frac{ \tail }{ N } } +
        \frac{ k d_2 + \log(N/\delta) }{ N_p }
    \right), \label{ineq:A.11}
\end{align}
where \(C' > 0\) is a constant. 

Notably, from \Cref{ineq:A.7.1}, by defining \(c = \frac{1+2c_3}{c_1 c_2}\) 
we have
\begin{align}
    \|\hat{w}_i - w_i^\diamond\|_{\Sigmav^\diamond}
    &\leq \sqrt{ c^2 \|\nabla f(w_i^\diamond)\|_{(\Sigmav_i^\diamond + \lambda\Iv)^{-1}}^2 + 2 c \lambda B \|\nabla f(w_i^\diamond)\|_{(\Sigmav_i^\diamond + \lambda\Iv)^{-1}} }. \label{ineq:A.12}
\end{align}

Therefore, by setting
\begin{align*}
    \lambda = \frac{c C'}{2B} \sqrt{
        \left\| \Thetav_i^\star - \Thetav_i^\diamond \right\|_F^2 +
        \frac{1}{N N_p \nu} \log\left( \frac{ \mathcal{N}_{\mathcal{G}'_{\rbm}}(1/(N N_p)) }{ \delta } \right) +
        \frac{1}{\nu} \sqrt{ \frac{ \tail }{ N } } +
        \frac{ k d_2 + \log(N/\delta) }{ N_p }
    },
\end{align*}
and by combining \Cref{ineq:A.11} with \Cref{ineq:C.14}, we obtain
\begin{align}
    &\|\hat{w}_i - w_i^\diamond\|_{\Sigmav^\diamond}^2 \\
    &\leq \sqrt{2} \, c C'
    \left(
        \left\| \Thetav_i^\star - \Thetav_i^\diamond \right\|_F^2 +
        \frac{1}{N N_p \nu} \log\left( \frac{ \mathcal{N}_{\mathcal{G}'_{\rbm}}(1/(N N_p)) }{ \delta } \right) +
        \frac{1}{\nu} \sqrt{ \frac{ \tail }{ N } } +
        \frac{ k d_2 + \log(N/\delta) }{ N_p }
    \right). \label{ineq:C.14}
\end{align}

Note that by combining \Cref{equ:A.4} with \Cref{def:c1} we have
\begin{align}
    \frac{1}{N_p} \sum_{j \in N_p} \Ac_{i,j}
    \leq 4LB^2 \left\| \Bv^\diamond - \hat{\Bv} \right\|^2 + 2c'_1 \left\| w_i^\diamond - \hat{w}_i \right\|^2_{\Sigmav^\diamond_i}.\label{ineq:C.14.1}
\end{align}

From \Cref{cor:label-is-correct}, we have
\begin{align}
    \left\| \Bv - \Bv^\diamond \right\|^2_F 
    \leq 2 \frac{C C_3}{NN_p \nu} \log\left( \frac{ \Nc_{\Gc'_{\rbm}}(1/(NN_p)) }{ \delta } \right) + 2 \frac{C C_4}{\nu} \kappa^2 \sqrt{ \frac{ \tail }{ N } }. \label{ineq:C.15}
\end{align}
Thus, by combining \Cref{ineq:C.14}, \Cref{ineq:C.14.1}, and \Cref{ineq:C.15}, we conclude that there exists a constant \(C_5 > 0\) such that
\begin{align}
    \frac{1}{N_p} \sum_{j \in N_p} \Ac_{i,j} 
    \leq C_5 \left(
        \left\| \Thetav_i^\star - \Thetav_i^\diamond \right\|^2 +
        \frac{1}{NN_p \nu} \log\left( \frac{ \Nc_{\Gc'_{\rbm}}(1/(NN_p)) }{ \delta } \right) +
        \frac{1}{\nu} \sqrt{ \frac{ \tail }{ N } } +
        \frac{ k d_2 + \log(N/\delta) }{ N_p }
    \right).\label{ineq:C.19.1}
\end{align}

\textbf{Step 2: Bounding term \(\Bc_{i}\).}

Note that from \Cref{ineq:C.8}, we have
\begin{align*}
    \|\hat{w}_i - w_i^\diamond\|^2
    \leq \frac{2}{L_2} \nabla f(w_i^\diamond)^\top (w_i^\diamond - \hat{w}_i) \leq \|\nabla f(w_i^\diamond)\|_{(\Sigmav_i^\diamond + \lambda \Iv)^{-1}} \|w_i^\diamond - \hat{w}_i\|_{\Sigmav_i^\diamond + \lambda \Iv}.
\end{align*}
Therefore, we obtain
\begin{align*}
    \|\hat{w}_i - w_i^\diamond\|^2 
    \leq \sqrt{ \lambda^2 \|\nabla f(w_i^\diamond)\|_{(\Sigmav_i^\diamond + \lambda \Iv)^{-1}}^2 + 2 \|\nabla f(w_i^\diamond)\|_{(\Sigmav_i^\diamond + \lambda \Iv)^{-1}} \|w_i^\diamond - \hat{w}_i\|_{\Sigmav_i^\diamond} }.
\end{align*}
Let
\begin{align*}
    \lambda \leq \min\left\{ 1,\ \frac{c C'}{2 B} \sqrt{
        \left\| \Thetav_i^\star - \Thetav_i^\diamond \right\|_F^2 +
        \frac{1}{N N_p \nu} \log\left( \frac{ \mathcal{N}_{\mathcal{G}'_{\rbm}}(1/(N N_p)) }{ \delta } \right) +
        \frac{1}{\nu} \sqrt{ \frac{ \tail }{ N } } +
        \frac{ k d_2 + \log(N/\delta) }{ N_p }
    } \right\},
\end{align*}
using \Cref{ineq:A.11} and \Cref{ineq:A.12}, we obtain
\begin{align}
    \|\hat{w}_i - w_i^\diamond\|^2
    \leq C_6
    \left(
        \left\| \Thetav_i^\star - \Thetav_i^\diamond \right\|_F^2 +
        \frac{1}{N N_p \nu} \log\left( \frac{ \mathcal{N}_{\mathcal{G}'_{\rbm}}(1/(N N_p)) }{ \delta } \right) +
        \frac{1}{\nu} \sqrt{ \frac{ \tail }{ N } } +
        \frac{ k d_2 + \log(N/\delta) }{ N_p }
    \right), \label{ineq:C.19}
\end{align}
for a constant \(C_6 \geq 0\). Leveraging the smoothness of \(r_{\theta}(\cdot)\) with respect to \(\theta\), we obtain the bound
\begin{align*}
    \Bc_{i,j}
    &\leq \left(\|\hat{\theta}_i-(\ikp\hat{\Bv})w_i^\diamond\|+\|(\ikp\hat{\Bv})w_i^\diamond-\theta^\diamond_i\|\right)^2\\
    &\leq
    \left(\|\hat{w}_i-w^\diamond_i\|+B\|\hat{\Bv}-\Bv^\diamond\|\right)^2.
\end{align*}
Therefore, applying \Cref{cor:5.2} and using \Cref{ineq:C.19} we obtain 
\begin{align}
    \Bc_{i,j}\leq C_7
    \left(
        \left\| \Thetav_i^\star - \Thetav_i^\diamond \right\|_F^2 +
        \frac{1}{N N_p \nu} \log\left( \frac{ \mathcal{N}_{\mathcal{G}'_{\rbm}}(1/(N N_p)) }{ \delta } \right) +
        \frac{1}{\nu} \sqrt{ \frac{ \tail }{ N } } +
        \frac{ k d_2 + \log(N/\delta) }{ N_p }
    \right)\label{ineq:C.20}
\end{align}

\textbf{Step 3: Putting \(\Ac_{i,j}\) and \(\Bc_{i,j}\) together.}

Combining \Cref{ineq:C.4.1}, \Cref{ineq:C.19.1} and \Cref{ineq:C.20} we have
\begin{align*}
    &\frac{1}{N_p}\sum_{j\in[N_p]}\left|(r_{\theta^\diamond_i}(\tau_{i, 0}^{(j)}) - r_{\theta^\diamond_i}(\tau_{i, 1}^{(j)})) - (r_{\hat{\theta}_i}(\tau_{i, 0}^{(j)}) - r_{\hat{\theta}_i}(\tau_{i, 1}^{(j)}))\right|^2\\
    &\quad\leq C_7
    \left(
        \left\| \Thetav_i^\star - \Thetav_i^\diamond \right\|_F^2 +
        \frac{1}{N N_p \nu} \log\left( \frac{ \mathcal{N}_{\mathcal{G}'_{\rbm}}(1/(N N_p)) }{ \delta } \right) +
        \frac{1}{\nu} \sqrt{ \frac{ \tail }{ N } } +
        \frac{ k d_2 + \log(N/\delta) }{ N_p }
    \right),
\end{align*}
where \(C_7>0\) is a constant. Therefore, combining with \Cref{equ:C.1} we obtain
\begin{align*}
        &\frac{1}{N_p}\sum_{j\in[N_p]}\left|(r_{\hat{\Thetav}_i}(\tau_{i, 0}^{(j)}) - r_{\hat{\Thetav}_i}(\tau_{i, 1}^{(j)})) - (r_{\Thetav_i^\star}(\tau_{i, 0}^{(j)}) - r_{\Thetav_i^\star}(\tau_{i, 1}^{(j)}))\right|^2 \\
    &\quad\leq \frac{1}{N_p}\sum_{j\in[N_p]}\left|(r_{\hat{\theta}_i}(\tau_{i, 0}^{(j)}) - r_{\hat{\theta}_i}(\tau_{i, 1}^{(j)})) - (r_{\theta_i^\diamond}(\tau_{i, 0}^{(j)}) - r_{\theta_i^\diamond}(\tau_{i, 1}^{(j)}))\right|^2 + 4L_1\|\Thetav_i^\star - \Thetav_i^\diamond\|^2\\
    &\quad\leq C_8\left(
        \left\| \Thetav_i^\star - \Thetav_i^\diamond \right\|_F^2 +
        \frac{1}{N N_p \nu} \log\left( \frac{ \mathcal{N}_{\mathcal{G}'_{\rbm}}(1/(N N_p)) }{ \delta } \right) +
        \frac{1}{\nu} \sqrt{ \frac{ \tail }{ N } } +
        \frac{ k d_2 + \log(N/\delta) }{ N_p }
    \right),
\end{align*}
where \(C_8>0\) is a constant. The proof is thus complete.
\end{proof}

Also, we introduce the following short lemma to upper bound the expected squared difference between the true reward differences and their estimates use \Cref{thm:C.1}.
\begin{lemma}\label{lemma:lemma6}
    Assume \Cref{assum:2} holds. For any \(\delta\in(0,1]\), if \(N\geq N_{\text{unif}}(\Gc_r,\mu_0,\mu_1,\delta)\), with probability at least \(1-\delta\), we have
    \begin{align*}
    &\mathop{\Eb}_{\tau_0\sim\mu_0, \tau_1\sim\mu_1}\left[\left|(r_{\theta^\star_i}(\tau_{0}) - r_{\theta^\star_i}(\tau_{1})) - (r_{\hat{\theta}_i}(\tau_{0}) - r_{\hat{\theta}_i}(\tau_{1}))\right|^2\right]\\
    &\quad\leq C_9
    \left(
        \left\| \Thetav_i^\star - \Thetav_i^\diamond \right\|_F^2 +
        \frac{1}{N N_p \nu} \log\left( \frac{ \mathcal{N}_{\mathcal{G}'_{\rbm}}(1/(N N_p)) }{ \delta } \right) +
        \frac{1}{\nu} \sqrt{ \frac{ \tail }{ N } } +
        \frac{ k d_2 + \log(N/\delta) }{ N_p }
    \right),
\end{align*}
where \(C_8> 0\) is a constant.
\end{lemma}
\begin{proof}
From \Cref{assum:2}, if \(N\geq N_{\text{unif}}(\Gc_r,\mu_0,\mu_1,\delta)\), with probability at least \(1-\delta\), we have
\begin{align*}
    &\mathop{\Eb}_{\tau_0\sim\mu_0, \tau_1\sim\mu_1}\left[\left|(r_{\theta^\star_i}(\tau_{0}) - r_{\theta^\star_i}(\tau_{1})) - (r_{\hat{\theta}_i}(\tau_{0}) - r_{\hat{\theta}_i}(\tau_{1}))\right|^2\right]\\
    &\quad\leq
    \frac{1.1}{N_p}\sum_{j\in[N_p]}\left|(r_{\theta^\star_i}(\tau_{i, 0}^{(j)}) - r_{\theta^\star_i}(\tau_{i, 1}^{(j)})) - (r_{\hat{\theta}_i}(\tau_{i, 0}^{(j)}) - r_{\hat{\theta}_i}(\tau_{i, 1}^{(j)}))\right|^2\\
    &\quad\leq C_8
    \left(
        \left\| \Thetav_i^\star - \Thetav_i^\diamond \right\|_F^2 +
        \frac{1}{N N_p \nu} \log\left( \frac{ \mathcal{N}_{\mathcal{G}'_{\rbm}}(1/(N N_p)) }{ \delta } \right) +
        \frac{1}{\nu} \sqrt{ \frac{ \tail }{ N } } +
        \frac{ k d_2 + \log(N/\delta) }{ N_p }
    \right),
\end{align*}
where \(C_8> 0\) is a constant.  This is the desired result.
\end{proof}

With the assistance of \Cref{lemma:lemma6}, we are now prepared to prove \Cref{thm:diverse-newmodel}.

\diverse*
\begin{proof}
To simplify notation, let \( C_{\rbm} = C_{\rbm}(\Gc_{\rbm}, \pi_{i, \text{tar}}, \mu_{i, \text{ref}}, i) \). Following the approach in \citet{park2024rlhf}, define \( r_{\pi}^{i, \text{inf}} := \argmin_{\rbm \in \Rc_i} \left(J(\pi, r_i) - \mathbb{E}_{\tau \sim \mu_{i, \text{ref}}}[r_i(\tau)]\right) \). By the continuity of \( r \) and the definition of \(\Rc_i\), for any policy \(\pi\), we have
\begin{align*}
    |\hat{r}_{i}(\tau_{i, 1}) -\hat{r}_{i}(\tau_{i, 0})) - (r_{ \pi}^{i, \text{inf}}(\tau_{i, 1}) -r_{ \pi}^{i, \text{inf}}(\tau_{i, 0})|\leq L_1\zeta.
\end{align*}
Thus, it follows that
    \begin{equation*}
    \begin{aligned}
        &J(\pi_{i, \text{tar}}; r^\star_i) - J(\hat{\pi}_i; r^\star_i) 
        \\
        &=(J(\pi_{i, \text{tar}}; r^\star_i) - \Eb_{\tau \sim \mu_{i, \text{ref}}}[r^\star_i(\tau)]) - (J(\hat{\pi}_i; r^\star_i) -\Eb_{\tau \sim \mu_{i, \text{ref}}}[r^\star_i(\tau)])
        \\
        &\leq (J(\pi_{i, \text{tar}}; r^\star_i) - \Eb_{\tau \sim \mu_{i, \text{ref}}}[r^\star_i(\tau)])   - (J(\pi_{i, \text{tar}}; r_{ \pi_{i, \text{tar}}}^{i, \text{inf}}) - \Eb_{\tau \sim \mu_{i, \text{ref}}}[r_{\pi_{i, \text{tar}}}^{i, \text{inf}}(\tau)]) 
        \\
        &\qquad
        + (J(\hat{\pi}_{i}; r_{\hat{\pi}_i}^{i, \text{inf}}) - \Eb_{\tau \sim \mu_{i, \text{ref}}}(r_{\hat{\pi}_i}^{i, \text{inf}}(\tau)))- (J(\hat{\pi}_i; r^\star_i) -\Eb_{\tau \sim \mu_{i, \text{ref}}}[r^\star_i(\tau)])
        \\
        &\leq (J(\pi_{i, \text{tar}}; r^\star_i) - \Eb_{\tau \sim \mu_{i, \text{ref}}}[r^\star_i(\tau)])  - (J(\pi_{i, \text{tar}}; r_{ \pi_{i, \text{tar}}}^{i, \text{inf}}) - \Eb_{\tau \sim \mu_{i, \text{ref}}}[r_{\pi_{i, \text{tar}}}^{i, \text{inf}}(\tau)])
        \\
        &\leq \Eb_{\tau_{i, 0} \sim \pi_{i, \text{tar}}, \tau_{i, 1} \sim \mu_{i, \text{ref}}}[(r_{i}^\star(\tau_{i, 1}) -r_{i}^\star(\tau_{i, 0})) - (r_{ \pi_{i, \text{tar}}}^{i, \text{inf}}(\tau_{i, 1}) -r_{ \pi_{i, \text{tar}}}^{i, \text{inf}}(\tau_{i, 0}))]+L_1\zeta
        \\
        &\leq C_{\rbm}\sqrt{\Eb_{\mu_0, \mu_1}\left[ \big| {(r_{i}^\star(\tau_{i, 1}) -r_{i}^\star(\tau_{i, 0})) - (\hat{r}_{i}(\tau_{i, 1}) -\hat{r}_{i}(\tau_{i, 0}))}\big|^2\right]}+L_1\zeta
       \\
        &\leq \sqrt{C  C_{\rbm}^2 \left(
        \left\| \Thetav_i^\star - \Thetav_i^\diamond \right\|_F^2 +
        \frac{1}{N N_p \nu} \log\left( \frac{ \mathcal{N}_{\mathcal{G}'_{\rbm}}(1/(N N_p)) }{ \delta } \right) +
        \frac{1}{\nu} \sqrt{ \frac{ \tail }{ N } } +
        \frac{ k d_2 + \log(N/\delta) }{ N_p }
    \right)} 
    \end{aligned} 
\end{equation*}
where $C>0$ is a constant. The proof is thus complete.
\end{proof}

\subsection[Proof of Corollary~\ref{thm:averaged-gap}]{Proof of \Cref{thm:averaged-gap}}

\average*

\begin{proof}
    From \Cref{thm:diverse-newmodel}, by summing over \(i \in [N]\), we obtain the following inequality:
    \begin{align*}
        &\frac{1}{N} \sum_{i \in [N]} J(\pi_{i, \text{tar}}; r^\star_i) - J(\hat{\pi}_i'; r^\star_i) \\
        &\leq \sqrt{c_3 \left( \frac{1}{N} \sum_{i \in [N]} \left\| \Thetav_i^\star - \Thetav_i^\diamond \right\|_F^2 +
        \frac{1}{N N_p \nu} \log\left( \frac{ \mathcal{N}_{\mathcal{G}'_{\rbm}}(1/(N N_p)) }{ \delta } \right) +
        \frac{1}{\nu} \sqrt{ \frac{ \tail }{ N } } +
        \frac{ k d_2 + \log(N/\delta) }{ N_p }
    \right)}.
    \end{align*}
    
    Furthermore, we can derive the following bound:
    \begin{align*}
        &\frac{1}{N} \sum_{i \in [N]} J(\pi_{i, \text{tar}}; r^\star_i) - J(\hat{\pi}_i'; r^\star_i) \\
        &\leq
        \sqrt{c_4 \left( \frac{1}{N N_p \nu} \log\left( \frac{\mathcal{N}_{\mathcal{G}'_{\rbm}}(1/(N N_p))}{\delta} \right) + \frac{1}{\nu} \sqrt{\frac{\tail}{N}} + \frac{k d_2 + \log(N/\delta)}{N_p} \right)}.
    \end{align*}
    
    Here, the inequality holds because \(\frac{\tail}{N} \leq \frac{1}{\mu} \sqrt{\frac{\tail}{N}}\) for \(N \geq \mu^2 \tail\).  This proves the corollary.
\end{proof}

\section[Deferred Proofs in~\ref{sec:results}]{Deferred Proofs in \Cref{sec:results}}\label{apx:D}

% \subsection[Proof of Corollary~\ref{thm:average-sft}]{Proof of \Cref{thm:average-sft}}

First, we introduce two auxiliary lemmas.
\begin{lemma}
For reward function \(r\), suppose \Cref{assumption:reward} holds. Then, we have
    \begin{align*}
    \frac{1}{N} \sum_{i \in [N]} \left|\log\Phi(r_i^\star(\tau_0) - r_i^\star(\tau_1)) - \log\Phi(r_{\Thetav^{\init}+\Thetav_i^\diamond}(\tau_0) + r_{\Thetav^{\init}+\Thetav_i^\diamond}(\tau_1))\right|
    \leq 2LL' \sqrt{\frac{\tail}{N}}.
\end{align*}
\end{lemma}
\begin{proof}
 From the \(L\)-Lipschitz continuity of the function \(\log\Phi(x)\), for any trajectories \(\tau_0\) and \(\tau_1\), we have
\begin{align*}
    &\left|\log\Phi(r_i^\star(\tau_0) - r_i^\star(\tau_1)) - \log\Phi(r_{\Thetav^{\init}+\Thetav_i^\diamond}(\tau_0) - r_{\Thetav^{\init}+\Thetav_i^\diamond}(\tau_1))\right|\\
    &\leq L \left|r_i^\star(\tau_0) - r_i^\star(\tau_1) - r_{\Thetav^{\init}+\Thetav_i^\diamond}(\tau_0) + r_{\Thetav^{\init}+\Thetav_i^\diamond}(\tau_1)\right|.
\end{align*}
From the \(L'\)-Lipschitz continuity of the function \(r(\tau;\Thetav)\) with respect to \(\Thetav\), we have
\begin{align*}
    \left|r_i^\star(\tau_0) - r_i^\star(\tau_1) - r_{\Thetav^{\init}+\Thetav_i^\diamond}(\tau_0) + r_{\Thetav^{\init}+\Thetav_i^\diamond}(\tau_1)\right|
    \leq 2L' \|\Thetav_i^\star -\Thetav^{\init}-\Thetav_i^\diamond\|_F.
\end{align*}
Therefore, we have
\begin{align*}
    &\frac{1}{N}\sum_{i \in [N]} \left|r_i^\star(\tau_0) - r_i^\star(\tau_1) - r_{\Thetav^{\init}+\Thetav_i^\diamond}(\tau_0) + r_{\Thetav^{\init}+\Thetav_i^\diamond}(\tau_1)\right|\\
    &\leq \sqrt{\frac{1}{N} \sum_{i \in [N]} \left(r_i^\star(\tau_0) - r_i^\star(\tau_1) - r_{\Thetav^{\init}+\Thetav_i^\diamond}(\tau_0) + r_{\Thetav^{\init}+\Thetav_i^\diamond}(\tau_1)\right)^2} \\
    &\leq 2L' \sqrt{\frac{1}{N} \|\Thetav^\star - \Thetav^{\init,(N)}-\Thetav^\diamond\|_F^2}\\
    & = 2L' \sqrt{\frac{1}{N} \|\Delta\Thetav^\star-\Thetav^\diamond\|_F^2}.
\end{align*}
Note that \(\Thetav^\diamond\) can be derived from the truncated SVD of \(\Delta\Thetav^\star\), retaining the top \(k\) singular values \citep{golub2013matrix,liu2024federated}. Consequently, we have
\begin{align*}
    \frac{1}{N} \sum_{i \in [N]} \left|r_i^\star(\tau_0) - r_i^\star(\tau_1) - r_{\Thetav^{\init}+\Thetav_i^\diamond}(\tau_0) + r_{\Thetav^{\init}+\Thetav_i^\diamond}(\tau_1)\right|
    \leq 2L' \sqrt{\frac{\tail}{N}}.
\end{align*}
Then we obtain
\begin{align*}
    \frac{1}{N} \sum_{i \in [N]} \left|\log\Phi(r_i^\star(\tau_0) - r_i^\star(\tau_1)) - \log\Phi(r_{\Thetav^{\init}+\Thetav_i^\diamond}(\tau_0) + r_{\Thetav^{\init}+\Thetav_i^\diamond}(\tau_1))\right|
    \leq 2LL' \sqrt{\frac{\tail}{N}},
\end{align*}
which completes the proof.
\end{proof}

\begin{lemma}\label{lemma:7}
For local reward models parameterized by { \(\{\Thetav_i\}_{i=1}^N\)} with { \(\Thetav_i = \Thetav^{\init}+\Delta\Thetav_i\)}, if there exists a constant \(\delta > 0\) such that
\begin{align*}
    \sum_{i \in [N]} \Eb_{\mu_0, \mu_1} \left[ \norm{P_{\Thetav_i} \left( \cdot \mid \tau_{i, 0}^{(j)}, \tau_{i,1}^{(j)} \right) - P_{\Thetav^\star_i} \left( \cdot \mid \tau_{i, 0}^{(j)}, \tau_{i,1}^{(j)} \right)}^2 \right] \leq \delta,
\end{align*}
then for \(\Bv\), \(\Bv^\diamond\in\Rb^{d_1\cdot k}\) with orthonormal columns satisfies \(\spn(\Bv)=\spn(\Delta\Thetav)\) and \(\spn(\Bv^\diamond) = \spn(\Delta\Thetav^\star)\), there exists a constant \(C>0\) such that
\begin{align*}
    \dist(\Bv, \Bv^\diamond) \leq \frac{\|\Delta\Thetav -(\Thetav^\star - \Thetav^{\init,(N)})\|_F^2}{(\delta')^2} \leq C\frac{\delta}{N\nu},
\end{align*}
where { \(\nu = \sigma_k\left(\frac{(\Delta\Thetav^\star)^T \Delta\Thetav^\star}{N}\right)\).}
\end{lemma}

\begin{proof}
From the Mean Value Theorem, there exists a constant \(C > 0\) such that 
\begin{align*}
    \|\Thetav - \Thetav^\star\|_F^2 \leq C \sum_{i \in [N]} \Eb_{\mu_0, \mu_1} \left[ \norm{P_{\Thetav_i} \left( \cdot \mid \tau_{i, 0}^{(j)}, \tau_{i,1}^{(j)} \right) - P_{\Thetav^\star_i} \left( \cdot \mid \tau_{i, 0}^{(j)}, \tau_{i,1}^{(j)} \right)}^2 \right] \leq C\delta.
\end{align*}
Define \(\delta' := \min_{1 \leq i \leq k, k+1 \leq j \leq \min\{d_1, Nd_2\}} \left|\sigma_i(\Delta\Thetav^\star) - \sigma_j(\Delta\Thetav)\right|\). Then, we observe that
\begin{align*}
    \delta' = \min_{1 \leq i \leq k, k+1 \leq j \leq \min\{d_1, Nd_2\}} \left|\sigma_i(\Delta\Thetav^\star) - \sigma_j(\Delta\Thetav)\right| = \sigma_k(\Delta\Thetav^\star).
\end{align*}
Next, by applying the Davis-Kahan Theorem, we obtain
\begin{align*}
    \text{dist}^2(\Bv, \Bv^\diamond) \leq \frac{\|\Delta\Thetav^\star - \Delta\Thetav\|_F^2}{(\delta')^2} \leq C \frac{\delta}{\sigma_k^2(\Delta\Thetav^\star)} = C \frac{\delta}{N\nu}.
\end{align*}
This proves the lemma.
\end{proof}

\subsection[Proof of Proposition~\ref{prop:1}]{Proof of \Cref{prop:1}}

In this section, we introduce the upper bound for the bracketing number of function class \(\mathcal{G}_{\mathbf{r}}(\Sc^{\text{ShareLoRA}})\), which is denoted by \(\Gc'_{\rbm}\).

\cond*
\begin{proof}
We start from the zero initialization case, therefore \(\Gc'_{\rbm}\) is equilvant to:
    \[
\mathcal{G}'_{\rbm} = \left\{ \left( r_{\boldsymbol{\Theta}_i}(\cdot) \right)_{i \in [N]} \,\Big|\, \Thetav \in \mathbb{R}^{d_1 \cdot N d_2},\ \text{rank}(\Thetav) = k,\ \|\Thetav_i\|_F \leq B,\ \forall i \in [N] \right\}.
\]
Similar to the proof in \citet[Proposition 1]{zhan2023provable}, we denote by \(\Fc\) the function class
\begin{align*}
    \Fc_{\rbm} = \left\{ \left( f_{i}(\cdot) \right)_{i \in [N]} \,\Big|\, f_i(\tau_0,\tau_1) = P_{r_i}(o=1 \mid \tau_0,\tau_1), \left( r_{i}(\cdot) \right)_{i \in [N]} \in \Gc'_{\rbm} \right\}.
\end{align*}
Let \(\Ic_{\Fc}(\epsilon)\) denote the \(\epsilon\)-bracket number with respect to the \(\ell_\infty\) norm. Therefore, there exist a set \(\bar{\Fc}\) satisfies \(|\bar{\Fc}|=\Ic_{\Fc}(\epsilon/4N)\) such that for any \(\left( f_{i}(\cdot) \right)_{i \in [N]}\in\Gc_{\rbm}\), there exist \(\left( \bar{f}_{i}(\cdot) \right)_{i \in [N]}\in \bar{\Fc}\) such that 
\begin{align*}
    \sup_{\tau_0,\tau_1}\left|f_i(\tau_0,\tau_1)-\bar{f}_i(\tau_0,\tau_1)\right|\leq\frac{\epsilon}{4N},\quad\forall i\in[N].
\end{align*}
Given \((\bar{f}_i)_{i\in[N]}\), construct a bracket \((g_{1}, g_{2})\):
\begin{align*}
    &[g_1(o=1|\tau_0,\tau_1)]_i=\bar{f}_i-\frac{\epsilon}{4N},\quad [g_1(o=0|\tau_0,\tau_1)]_i=1-\bar{f}_i-\frac{\epsilon}{4N},\\
    &[g_2(o=1|\tau_0,\tau_1)]_i=\bar{f}_i+\frac{\epsilon}{4N},\quad [g_2(o=0|\tau_0,\tau_1)]_i=1-\bar{f}_i+\frac{\epsilon}{4N}.
\end{align*}
Then, we observe that \((g_{1}, g_{2})\) satisfies \(g_1(\tau_0, \tau_1) \leq g_2(\tau_0, \tau_1)\), \(\| g_1(\tau_0, \tau_1) - g_2(\tau_0, \tau_1) \|_1 \leq \epsilon\) and \( g_{1}(\tau_0, \tau_1) \leq P_{\rbm}(\cdot \mid \tau_0, \tau_1) \leq g_{2}(\tau_0, \tau_1)\). Therefore, our goal is to bound \(\Ic_{\Fc}(\epsilon/4N)\). From the mean value theorem, for \(a,b\in[-2R,2R]\),  there exist constant \(C_{R} = \max_{a\in[-2R,2R]}|\Phi'(a)|\) such that 
\begin{align*}
    |\Phi(b)-\Phi(a)|\leq C_R |b-a|.
\end{align*}
Denote \(\fbm = [f_1,\cdots,f_N]^\top\), we obtain 
\begin{align}
    \left|\fbm_i(\tau_0,\tau_1)-\bar{\fbm}_i(\tau_0,\tau_1)\right|
    &\leq C_R|\rbm(\tau_0)-\rbm(\tau_1)-\rbm'(\tau_0)+\rbm'(\tau_1)|\nonumber\\
    &\leq 2C_RL_1\|\mtov(\Thetav)-\mtov(\Thetav')\|\nonumber\\
    &=2C_RL_1\left\|\diag(\Bv)\mtov(\Wv)-\diag(\Bv')\mtov(\Wv')\right\|\nonumber\\
    &\leq2C_RL_1\left\|\mtov(\Wv)-\mtov(\Wv')\right\|+2C_RL_1B\left\|\diag(\Bv)-\diag(\Bv')\right\|\nonumber\\
    &\leq \max\{1,B\}2C_RL_1\mleft\|\begin{bmatrix}
        \mtov(\Bv)\\
        \mtov(\Wv)
    \end{bmatrix}-\begin{bmatrix}
        \mtov(\Bv')\\
        \mtov(\Wv')
    \end{bmatrix}\mright\|.\label{ineq:B.25}
\end{align}

Denote \( C_R' = \max\{1, B\} \cdot 2C_R L_1 \). From \Cref{ineq:B.25}, we conclude that the \(\epsilon/4N\)-bracket number \(\Ic_{\Fc}(\epsilon/4N)\) is bounded by the \(\epsilon'\)-covering number of a \((d_1k + Nd_2k)\)-dimensional ball centered at the origin with radius \(B\) with respect to the \(\ell_2\) norm, where
\begin{align*}
    \epsilon' = \frac{\epsilon}{4N \cdot \max\{1, B\} \cdot 2C_R L_1}.
\end{align*}
According to \citet{wainwright2019high}, this covering number is upper bounded by
{ \(\Oc\left( (d_1k + Nd_2k) \log\left( \frac{N}{\epsilon} \right) \right)\)}.
Therefore, for \(\epsilon = 1/(NN_p)\), we conclude that the covering number \(\Nc_{\Gc'_{\rbm}}\left(1/(NN_p) \right)\) is upper bounded by { \(\Oc\left( (d_1k + Nd_2k) \log(NN_p) \right)\)}. We note that for \(\mathcal{G}'_{\mathbf{r}|\Thetav^\init}\), following the same proof process, we can show that 
\[\mathcal{G}'_{\mathbf{r}|\Thetav^\init}\leq\Oc\left( (d_1k + Nd_2k) .\log(NN_p) \right).\]
The proof is thus complete.
\end{proof}

\subsection[Proof of Theorem~\ref{cor:5.2}]{Proof of \Cref{cor:5.2}}

We note that the proof of \Cref{cor:5.2} is a natural extension of the argument used in \Cref{cor:label-is-correct} as detailed in \Cref{sec:apx.A.1}.
\corsft*
\begin{proof}
We define the event $\Ec_1, \Ec_2$ as satisfying \Cref{lemma:mle}, \Cref{lemma:l2distance} with $\delta \leftarrow \delta/2$, respectively, so we have $\Pb(\Ec_1\cap \Ec_2) > 1- \delta$. We will only consider the under event $\Ec_1\cap \Ec_2$. From \Cref{lemma:4}, we have
    \begin{align*}
        \sum_{i \in [N]} \sum_{j \in [N_{p}]}
    \log P_{\Thetav^\star_i}(o_i^{(j)} \mid \tau_{i, 0}^{(j)}, \tau_{i,1}^{(j)})\leq \sum_{i \in [N]} \sum_{j \in [N_{p}]}
    \log P_{\Thetav^\diamond_i+\Thetav^{\init}}(o_i^{(j)} \mid \tau_{i, 0}^{(j)}, \tau_{i,1}^{(j)})+cN_p\sqrt{N\tail},
    \end{align*}
    Then, from the definition of \(\hat{\Thetav}\), we have
    \begin{align*}
        \sum_{i \in [N]} \sum_{j \in [N_{p}]}
    \log P_{\Thetav^\star_i}(o_i^{(j)} \mid \tau_{i, 0}^{(j)}, \tau_{i,1}^{(j)})\leq \sum_{i \in [N]} \sum_{j \in [N_{p}]}
    \log P_{\hat{\Thetav}_i}(o_i^{(j)} \mid \tau_{i, 0}^{(j)}, \tau_{i,1}^{(j)})+cN_p\sqrt{N\tail}.
    \end{align*}
    Therefore, similar to the proof in \Cref{thm:averaged-gap}, we obtain that
\begin{equation}
\begin{aligned}
        &\frac{1}{N}\sum_{i \in [N]} \Eb_{\mu_0, \mu_1} \left[ \left|(r_{\Thetav_i}(\tau_{i, 0}) - r_{\Thetav_i}(\tau_{i, 1})) - (r_{i}^\star(\tau_{i, 0}) - r_{i}^\star(\tau_{i, 1}))\right|^2 \right] 
        \\
        &\leq\frac{ \kappa^2 }{N}\sum_{i \in [N]} \Eb_{\mu_0, \mu_1}\left[ \norm{P_{\Thetav_i} ( \cdot \mid \tau_{i, 0}^{(j)}, \tau_{i,1}^{(j)}, i) - P_{\Thetav_i^\star} ( \cdot \mid \tau_{i, 0}^{(j)}, \tau_{i,1}^{(j)}, i)}_1^2 \right] 
        \\
        &\leq\frac{C_3\kappa^2}{NN_p} \log(\Nc_{\Gc_{\rbm}}(1/(NN_p))/ \delta)+C_4\kappa^2\sqrt{\frac{\tail}{N}}.
\end{aligned}
\label{eqn:C.2}    
\end{equation}
Combining this with \Cref{lemma:7}, we obtain
\begin{align*}
\text{dist}^2(\Bv, \Bv^\diamond) \leq 
\frac{C C_3}{NN_p \nu} \log\left(\frac{\Nc_{\Gc_{\rbm}}(1/(NN_p))}{\delta}\right) + \frac{C C_4}{\nu} \kappa^2 \sqrt{\frac{\tail}{N}}.
\end{align*}
Additionally, we have
\begin{align*}
\|\Bv - \Bv^\diamond\|^2_F \leq \text{dist}^2(\Bv, \Bv^\diamond) \leq 
\frac{C C_3}{NN_p \nu} \log\left(\frac{\Nc_{\Gc_{\rbm}}(1/(NN_p))}{\delta}\right) + \frac{C C_4}{\nu} \kappa^2 \sqrt{\frac{\tail}{N}}.
\end{align*}
Hence, the proof is complete.
\end{proof}

\subsection[Proof of Theorem~\ref{thm:worst-case-sft}]{Proof of \Cref{thm:worst-case-sft}}

We set the tolerance level \(\zeta\) in \Cref{alg:personal} to satisfy
\begin{equation}
    \zeta \leq c_4 \sqrt{
        \min_{i \in [N]} \left\| \Thetav_i^\star - \Thetav_i^\diamond \right\|_F^2 +
        \frac{\log\left( \mathcal{N}_{\mathcal{G}'_{\mathbf{r}}}\left(\frac{1}{N N_p}\right) / \delta \right)}{N N_p \nu} +
        \frac{1}{\nu} \sqrt{ \frac{\tail}{N} } +
        \frac{ k d_2 + \log\left(\frac{N}{\delta}\right) }{ N_p }
    }, \label{def:zeta'}
\end{equation}
where \(c_4 > 0\) is a constant. In \Cref{thm:worst-case-sft}, we build upon the proof established for \Cref{thm:diverse-newmodel} in \Cref{sec:apx.A.1}.

\worstcasesft*

\begin{proof}
Recall that \(\Delta\Thetav^\star_i = \Thetav^\star_i-\Thetav^{\init}\) and \(\Thetav^\diamond = \argmin_{\Thetav:\rank(\Thetav)=k}\|\Thetav-\Delta\Thetav^\star\|\).
By leveraging the continuity of \(r_{\Thetav}(\cdot)\), we can establish the following inequality:
\begin{align*}
    &\left|(r_{\hat{\Thetav}_i}(\tau_{i, 0}^{(j)}) - r_{\hat{\Thetav}_i}(\tau_{i, 1}^{(j)})) - (r_{\Thetav_i^\star}(\tau_{i, 0}^{(j)}) - r_{\Thetav_i^\star}(\tau_{i, 1}^{(j)}))\right|^2 \\
    &\leq
    2\left|(r_{\hat{\Thetav}_i}(\tau_{i, 0}^{(j)}) - r_{\hat{\Thetav}_i}(\tau_{i, 1}^{(j)})) - (r_{\Thetav^{\init}+\Thetav_i^\diamond}(\tau_{i, 0}^{(j)}) - r_{\Thetav^{\init}+\Thetav_i^\diamond}(\tau_{i, 1}^{(j)}))\right|^2 + 4L\|\Delta\Thetav_i^\star - \Thetav_i^\diamond\|^2.
\end{align*}
Therefore, similar to our proof in \Cref{thm:C.1}, we will show that with probability at least \(1-\delta\), we have
\begin{align*}
        &J(\pi_{i, \text{tar}}; r_i) - J(\hat{\pi}_i; r^\star_i)\\ &\leq 
         c_2\sqrt{ \left(
        \left\| \Delta\Thetav_i^\star - \Thetav_i^\diamond \right\|_F^2 +
        \frac{\log\left( \mathcal{N}_{\mathcal{G}_{\rbm}}(1/(N N_p))/ \delta \right)}{N N_p \nu} +
        \frac{1}{\nu} \sqrt{ \frac{ \tail }{ N } } +
        \frac{ k d_2 + \log(N/\delta) }{ N_p }
    \right)} .
\end{align*}
This concludes the proof.
\end{proof}

\subsection[Proof of Corollary~\ref{thm:average-sft}]{Proof of \Cref{thm:average-sft}}

We set the tolerance level \(\zeta\) in \Cref{alg:personal} as defined in \Cref{def:zeta'}. Similarly, the proof of \Cref{thm:average-sft} follows that of \Cref{thm:averaged-gap} presented in \Cref{sec:apx.A.1}.

\averagesft*
\begin{proof}
    From \Cref{thm:worst-case-sft}, by summing over \(i \in [N]\), we obtain the following inequality:
    \begin{align*}
        &\frac{1}{N} \sum_{i \in [N]} J(\pi_{i, \text{tar}}; r^\star_i) - J(\hat{\pi}_i'; r^\star_i) \\
        &\leq c_2\sqrt{\left( \frac{1}{N} \sum_{i \in [N]} \left\| \Thetav_i^\star - \Thetav_i^\diamond \right\|_F^2 +
        \frac{1}{N N_p \nu} \log\left( \frac{ \mathcal{N}_{\mathcal{G}'_{\rbm}}(1/(N N_p)) }{ \delta } \right) +
        \frac{1}{\nu} \sqrt{ \frac{ \tail }{ N } } +
        \frac{ k d_2 + \log(N/\delta) }{ N_p }
    \right)}.
    \end{align*}
    
    Similar to the proof for \Cref{thm:averaged-gap}, we have
    \begin{align*}
        \frac{1}{N} \sum_{i \in [N]} J(\pi_{i, \text{tar}}; r^\star_i) - J(\hat{\pi}_i'; r^\star_i) \leq
        c_3\sqrt{ \left( \frac{1}{N N_p \nu} \log\left( \frac{\mathcal{N}_{\mathcal{G}'_{\rbm}}(1/(N N_p))}{\delta} \right) + \frac{1}{\nu} \sqrt{\frac{\tail}{N}} + \frac{k d_2 + \log(N/\delta)}{N_p} \right)}.
    \end{align*}
This is the desired result.
\end{proof}

\newpage
\section{Experiment Details}

\subsection{Algorithms}\label{apx:alg}

In this section, we present the practical algorithms used for empirical evaluation. \Cref{alg:lora} outlines the \texttt{P-ShareLoRA} algorithm with a warm-up phase. Notably, by setting the number of warm-up epochs \(T_w = 0\), \Cref{alg:lora} reduces to the vanilla \texttt{P-ShareLoRA} algorithm. Conversely, setting \(T_w = T_{\text{Global}}\) transforms the algorithm into \texttt{global-P-ShareLoRA}. We define the per-sample function as \(f(\Thetav; o, \tau_{0}, \tau_{1}) := \log P_{\Thetav}(o \mid \tau_{0}, \tau_{1})\).
 \begin{algorithm}[ht]
	\caption{\texttt{P-ShareLoRA} for RLHF (with warm-up)\label{alg:lora}}
	\begin{algorithmic}
	 \STATE \textbf{Input:} Pre-trained model parameters \( W \); Human preference dataset \(\hat{\Dc}\); Rank \( r \); Scheduled learning rate \( \eta^t \); Number of warm-up epochs \( T_w \); Number of epochs \( T \).
	 \STATE \textbf{Initialize:} Low-rank matrices \( A \in \mathbb{R}^{d \times r} \), \( B \in \mathbb{R}^{r \times d} \), \(B_i\in\mathbb{R}^{r \times d}\quad\forall i\in[N]\) (e.g., randomly or zeros).
	 \STATE \textbf{Freeze} pre-trained weights \( W \).
      \STATE \rp{\textbf{Warm-up phase:}}
	 \FOR{each epoch \( t = 1 \) to \( T_w \)}
	 	\FOR{each \( \{o_{i}^{(j)}, \tau_{i,0}^{(j)}, \tau_{i,1}^{(j)}\} \in \hat{\Dc} \)}
                    \STATE Compute \(f(W+A^tB^t; o_{i}^{(j)}, \tau_{i,0}^{(j)}, \tau_{i,1}^{(j)})\).
                    \STATE Update \( A^{t+1} \leftarrow A^{t} - \eta^t \nabla_{A^t} f \).
	 		\STATE Update \( B^{t+1} \leftarrow B^{t} - \eta^t \nabla_{B^t} f \).
	 	\ENDFOR
	 \ENDFOR
      \STATE \rp{\textbf{Running }\texttt{P-ShareLoRA}:}
      \STATE \textbf{Set} \(A^1 \leftarrow A^{T_w}\), \(B_i^1\leftarrow B^{T_w}\quad \forall i\in[N]\).
      \FOR{each epoch \( t = 1 \) to \( T \)}
	 	\FOR{each random sampled \( \{o_{i}^{(j)}, \tau_{i,0}^{(j)}, \tau_{i,1}^{(j)}\} \in \hat{\Dc} \)}
                    \STATE Compute \(f(W+A^tB_i^t; o_{i}^{(j)}, \tau_{i,0}^{(j)}, \tau_{i,1}^{(j)})\).
                    \STATE Update \( A^{t+1} \leftarrow A^{t} - \eta^t \nabla_{A^t} f \).
	 		\STATE Update \( B_i^{t+1} \leftarrow B_i^{t} - \eta^t \nabla_{B_i^t} f \).
	 	\ENDFOR
	 \ENDFOR
      \STATE \rp{\textbf{Policy optimization by} \texttt{PPO-Clip}~\citep{schulman2017proximal}:}
\STATE \textbf{Initialize} policy parameters for each agent \( \theta_i\), \( \forall i \in [N] \).
\FOR{each PPO iteration \( k = 1 \) to \( K \)}
    \FOR{each agent \( i = 1 \) to \( N \) \textbf{in parallel}}
        \STATE Collect a set of trajectories \( \Dc_i \) by running policy \( \pi_{\theta_i^t} \).
        \STATE Compute rewards \( r_t^{(i)} \) and advantage estimates \( \hat{A}_t^{(i)} \) using GAE.
        \STATE Compute the PPO surrogate loss:
        \[
        \Lc^{\text{CLIP}}_i(\theta_i) = \mathbb{E}_t \left[ \min\left( \rho_t^{(i)}(\theta^t_i) \hat{A}_t^{(i)}, \operatorname{clip}\left( \rho_t^{(i)}(\theta^t_i), 1 - \epsilon, 1 + \epsilon \right) \hat{A}_t^{(i)} \right) \right],
        \]
        where \( \rho_t^{(i)}(\theta_i) = \dfrac{\pi_{\theta_i}(a_t^{(i)} | s_t^{(i)})}{\pi_{\theta_i^\text{old}}(a_t^{(i)} | s_t^{(i)})} \).
        \STATE Update \(\theta_i^{t+1}\leftarrow\theta_i^t-\eta^t\nabla_{\theta}\Lc^{\text{CLIP}}_i\).
    \ENDFOR
\ENDFOR
\STATE \textbf{Output:} Fine-tuned model parameters for each reward model \( A^T, \{B_i^T\}_{i=1}^N \); Fine-tuned model parameters for each local policy \(\{\theta_i^K\}_{i=1}^N\).
	\end{algorithmic}
\end{algorithm}

We also detail the baseline algorithms \texttt{LoRA-global} and \texttt{LoRA-local} in \Cref{alg:lora-global} and \Cref{alg:lora-local} for comparison.
 \begin{algorithm}[ht]
	\caption{Baseline algorithm 1: \texttt{LoRA-global}\label{alg:lora-global}}
	\begin{algorithmic}
	 \STATE \textbf{Input:} Pre-trained model parameters \( W \); Human preference dataset \(\hat{\Dc}\); Rank \( r \); Learning rate \( \eta \); Number of epochs \( T_{\text{Global}} \).
	 \STATE \textbf{Initialize:} Low-rank matrices \( A \in \mathbb{R}^{d \times r} \), \( B \in \mathbb{R}^{r \times d} \) (e.g., randomly or zeros).
	 \STATE \textbf{Freeze} pre-trained weights \( W \).
	 \FOR{each epoch \( t = 1 \) to \( T_{\text{Global}} \)}
	 	\FOR{each random sampled \( \{o_{i}^{(j)}, \tau_{i,0}^{(j)}, \tau_{i,1}^{(j)}\} \in \hat{\Dc} \)}
                    \STATE Compute \(f(W+A^tB^t; o_{i}^{(j)}, \tau_{i,0}^{(j)}, \tau_{i,1}^{(j)})\).
                    \STATE Update \( A^{t+1} \leftarrow A^{t} - \eta \nabla_{A^t} f \).
	 		\STATE Update \( B^{t+1} \leftarrow B^{t} - \eta \nabla_{B^t} f \).
	 	\ENDFOR
	 \ENDFOR
  \STATE \textbf{Output:} Fine-tuned model parameters for a global reward model \( A^{T_{\text{Global}}}, B^{T_{\text{Global}}} \).
	\end{algorithmic}
\end{algorithm}

 \begin{algorithm}[ht]
	\caption{Baseline algorithm 2: \texttt{LoRA-local}\label{alg:lora-local}}
	\begin{algorithmic}
	 \STATE \textbf{Input:} Pre-trained model parameters \( W \); Human preference dataset \(\hat{\Dc}\); Rank \( r \); Learning rate \( \eta \); Number of epochs \( T_{\text{local}} \).
	 \STATE \textbf{Initialize:} Low-rank matrices \( A_i \in \mathbb{R}^{d \times r} \), \( B_i \in \mathbb{R}^{r \times d}\quad \forall i\in[N] \) (e.g., randomly or zeros).
	 \STATE \textbf{Freeze} pre-trained weights \( W \).
      \FOR{each agent \( i = 1 \) to \( N \) \textbf{in parallel}}
	 \FOR{each epoch \( t = 1 \) to \( T_{\text{local}} \)}
	 	\FOR{each random sampled \( \{o_{i}^{(j)}, \tau_{i,0}^{(j)}, \tau_{i,1}^{(j)}\} \in \hat{\Dc}_i \)}
                    \STATE Compute \(f(W+A_i^tB_i^t; o_{i}^{(j)}, \tau_{i,0}^{(j)}, \tau_{i,1}^{(j)})\).
                    \STATE Update \( A_i^{t+1} \leftarrow A^{t} - \eta \nabla_{A^t} f \).
	 		\STATE Update \( B_i^{t+1} \leftarrow B^{t} - \eta \nabla_{B^t} f \).
	 	\ENDFOR
      \ENDFOR
	 \ENDFOR
  \STATE \textbf{Output:} Fine-tuned model parameters for each reward model \( \{A_i^{T_{\text{local}}}\}_{i=1}^N, \{B_i^{T_{\text{local}}}\}_{i=1}^N \).
	\end{algorithmic}
\end{algorithm}

\subsection{Implementation Details}\label{apx:impl}

\textbf{Hyperparamters.}
For all experiments conducted using both Vanilla LoRA (\texttt{LoRA-global} and \texttt{LoRA-local}) and \texttt{P-ShareLoRA} based algorithms, we employed a batch size of 128. The initial learning rate was set to $5 \cdot 10^{-5}$, with a linear scheduler applied to adjust the learning rate during training. For both GPT-J 6B and Llama3 8B models, the maximum token length was set to 2048. The rank \(k\) in all LoRA modules was fixed at 32, and the scaling factor $\alpha$ was set to 16. To simplify training, we applied LoRA only to the Q (query) and K (key) matrices for both models.

In the case of the \texttt{P-ShareLoRA(G)}, the initialization process was critical for ensuring effective fine-tuning. Specifically, the personalized $A$ matrices and the shared $B$ matrix were initialized using the $A$ and $B$ matrices obtained after two epochs of training with the \texttt{LoRA-global} method. Following this initialization, the PLAS model was fine-tuned for an additional epoch to refine the parameters further.

To maintain a fair comparison between \texttt{P-ShareLoRA(G)} and the other training methods, we adjusted the starting learning rate for \texttt{P-ShareLoRA(G)}. Given that a learning rate scheduler was used, the initial learning rate for PLAS-FT was set to one-third of the original learning rate, specifically $1.67 \cdot 10^{-5}$. This adjustment ensures that the fine-tuning process operates under comparable training dynamics as the baseline methods.

All experiments were implemented based on TRL\footnote{\href{https://huggingface.co/docs/trl/en/index}{https://huggingface.co/docs/trl/en/index}}, and additional hyperparameters were kept consistent across different methods.

\textbf{Computational Resources.} Our experiments were conducted using two NVIDIA A100 80GB GPUs. Training \texttt{P-LoRAShare(SI)} on a single GPU took around six hours, but this time could be reduced with multi-GPU training.

\subsection{Additional Experiment Results}\label{apx:D.3}

\textbf{Individual Labeler Performance.} In \Cref{sec:result}, we present the averaged preference estimation accuracy across all five labelers. In this section, we also provide the results of the separate estimation accuracy for each labeler in \Cref{fig:fig2}. We observe that our proposed methods, \texttt{P-ShareLoRA(SI)}, \texttt{P-ShareLoRA(G)}, and \texttt{P-ShareLoRA(WU)}, consistently outperform the baseline methods \texttt{LoRA-global} and \texttt{LoRA-local} for most of the labelers. Specifically, \texttt{P-ShareLoRA(WU)} achieves the highest accuracy for most labelers, peaking at 0.7803 for Labeler 1.  While \texttt{P-ShareLoRA(SI)} and \texttt{P-ShareLoRA(G)} also show significant improvements over the baseline methods for labelers 0,1 and 2.

\begin{figure}[h]
\centering
\includegraphics[width=0.65\textwidth]{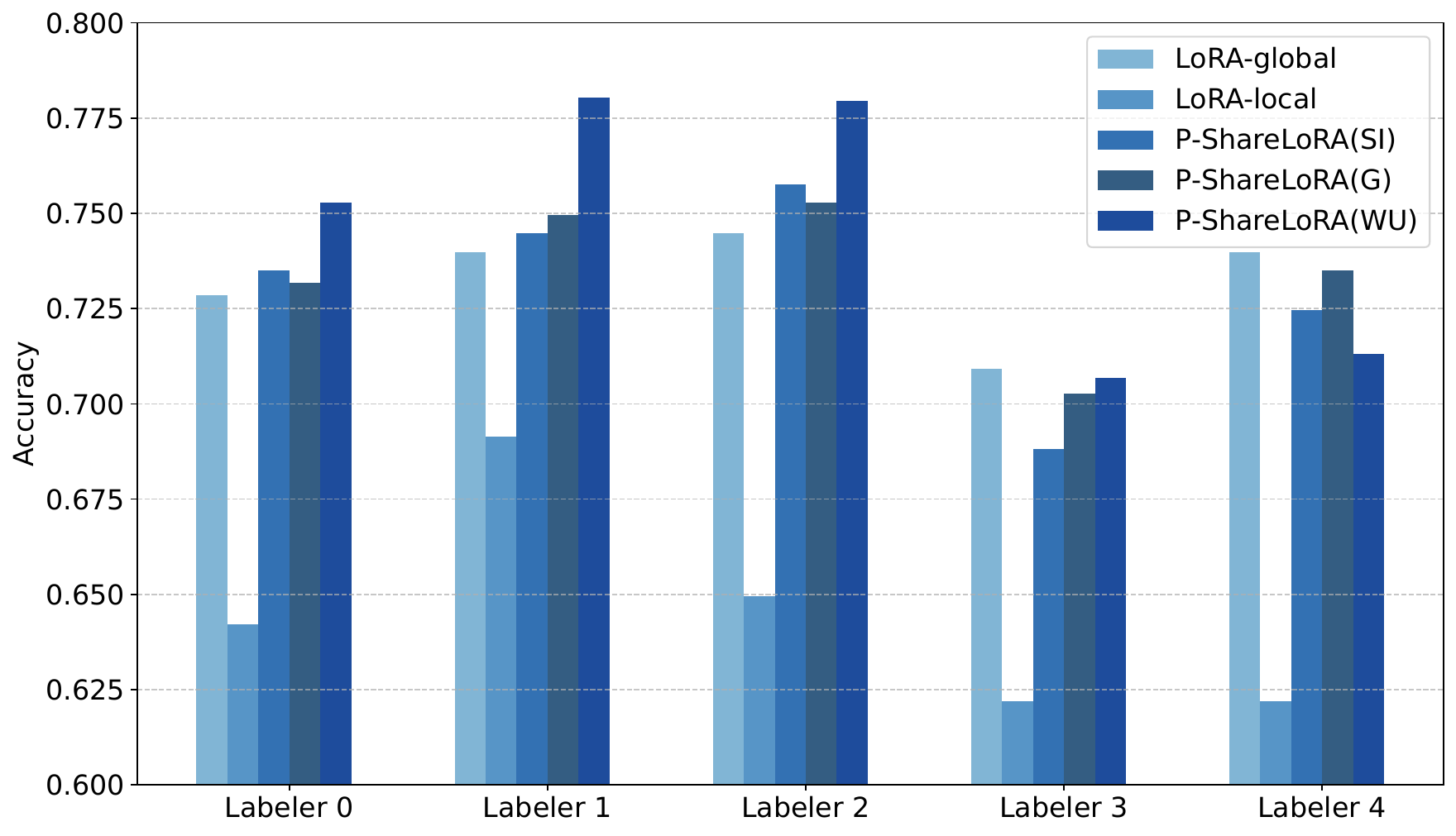}
\caption{Accuracies of Different Methods Across Labelers (Llama3 8B)}
\label{fig:fig2}
\end{figure}

\textbf{Share Down-projection VS Share Up-projection.}
Previous works \citep{tian2024hydralora,guo2024selective} have observed that the cosine similarity among down-projection matrices (\(A\) matrices) is significantly higher than that among up-projection matrices (\(B\) matrices). They interpret this as indicating that the down-projection matrices serve as a shared representation, mapping the input into a common representation space. Based on this observation, they introduce methods of sharing down-projection matrices among clients or experts. In contrast, our study finds that sharing the up-projection matrices (\(B\) matrices) yields better performance, as illustrated in Figure~\ref{fig:fig5}. Specifically, the approach of sharing \(B\) matrices consistently outperforms the method of sharing \(A\) matrices across all labelers and for both GPT-J 6B and Llama 3 8B models.

\begin{figure}[h]
\centering
\includegraphics[width=0.65\textwidth]{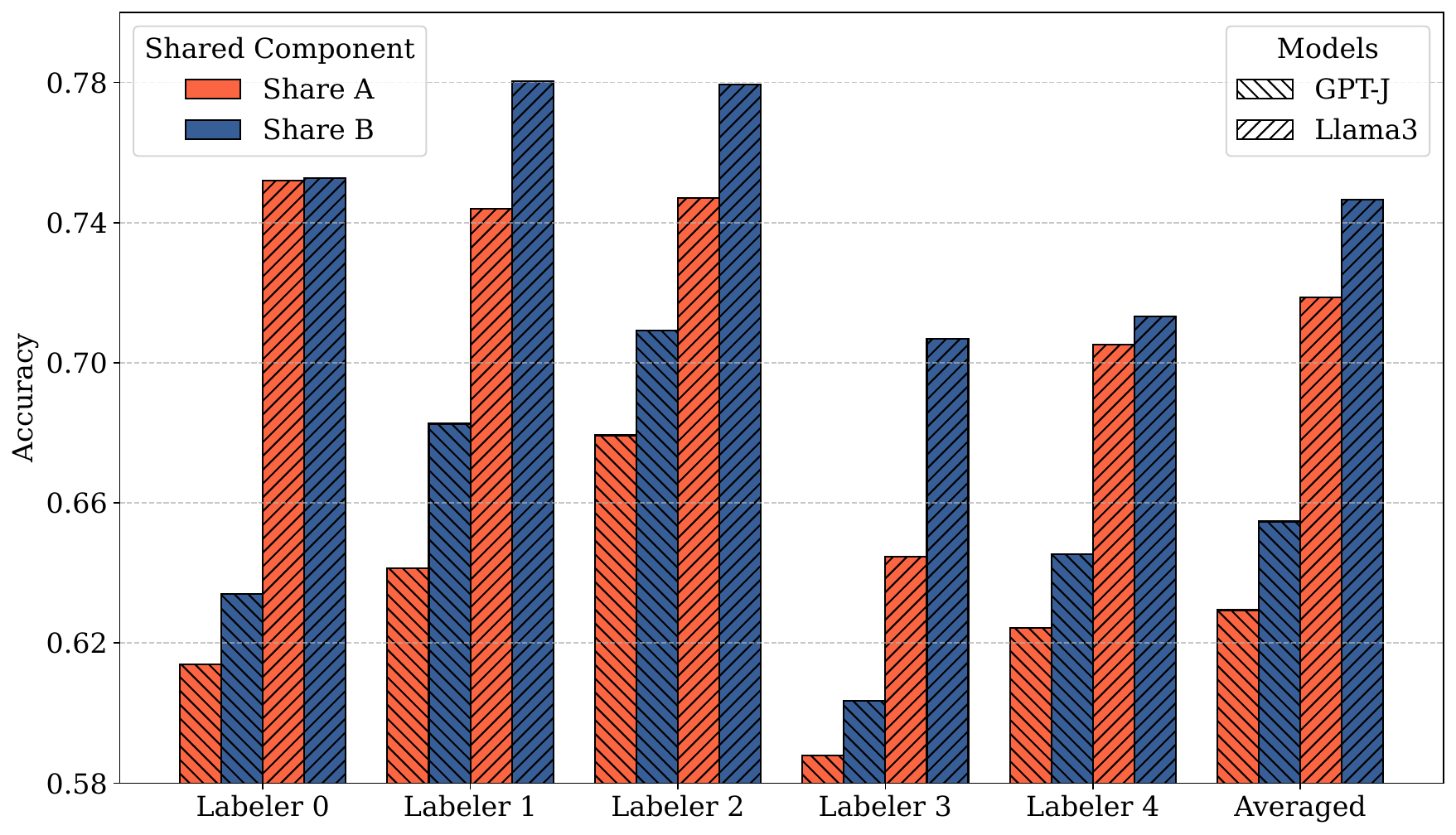}
\caption{Compare Accuracy between Share A and Share B}
\label{fig:fig5}
\end{figure}

\end{document}